\documentclass[11pt]{article}\usepackage{graphicx} 
\usepackage{fullpage}

\usepackage{hyperref}
\usepackage[utf8]{inputenc} 
\usepackage[T1]{fontenc}    
\usepackage{url}            
\usepackage{mathtools}
\usepackage{booktabs}       
\usepackage{amsfonts}       
\usepackage{nicefrac}       
\usepackage{microtype}      
\usepackage{xcolor}         
\usepackage{MnSymbol}
\usepackage{xspace}

\usepackage{tikz}
\usetikzlibrary{decorations.pathreplacing}

\usepackage{pict2e,picture,graphicx}
\usepackage[overload]{empheq}
\usepackage{microtype}
\usepackage{graphicx}
\usepackage{subcaption}
\usepackage{booktabs}
\usepackage{amsmath}
\usepackage{cases}
\usepackage{mathtools}
\usepackage{amsthm}
\usepackage{thm-restate}
\usepackage{enumitem}

\allowdisplaybreaks
\usepackage{xparse}
\usepackage[capitalize, noabbrev]{cleveref}
\usepackage{wrapfig}
\usepackage{bbm}
\usepackage{yfonts}
\usepackage{nicefrac} 
\usepackage{multirow}
\usepackage{bm}
\usepackage{bbm}
\usepackage{subcaption}

\hypersetup{
	colorlinks = true,
	linkcolor = black,
	citecolor = black,
	linktocpage = true,
	urlcolor = darkGreen
}
\usepackage{comment}

\theoremstyle{plain}
\newtheorem{theorem}{Theorem}[section]

\newtheorem{lemma}[theorem]{Lemma}

\newtheorem{corollary}[theorem]{Corollary}
\newtheorem*{theorem-non}{Theorem}

\theoremstyle{definition}

\theoremstyle{remark}
\newtheorem{remark}[theorem]{Remark}

	\definecolor{niceRed}{HTML}{BE2626}
	\definecolor{Red2}{HTML}{DB3236}
	\definecolor{mgreen}{HTML}{A0C88C}
	\definecolor{blueGrotto}{HTML}{059DC0}
	\definecolor{limeGreen}{HTML}{81B622}
	\definecolor{myellow}{HTML}{E09B23}
	\definecolor{darkGreen}{HTML}{2E8B57}
	\definecolor{navyBlueP}{HTML}{03468F}
	\definecolor{Sepia}{HTML}{7F462C}
	\definecolor{orange2}{HTML}{FF8000}
	\definecolor{mgray}{HTML}{ABB3B8}
	\definecolor{lgray}{HTML}{E5E8E9}
	\definecolor{myPurple}{HTML}{AF007C}
	\definecolor{mypurple2}{HTML}{CC9EFF}
	\definecolor{royalBlue}{HTML}{057DCD}
	\definecolor{mpink}{HTML}{FC6C85}
	\definecolor{lblue}{HTML}{4A90E2}
	\definecolor{peagreen}{HTML}{98C127}
	\definecolor{typnavy}{HTML}{001F3F}
	\definecolor{typblue}{HTML}{0074D9}
	\definecolor{typaqua}{HTML}{7FDBFF}
	\definecolor{typteal}{HTML}{39CCCC}
	\definecolor{typeastern}{HTML}{239DAD}
	\definecolor{typpurple}{HTML}{B10DC9}
	\definecolor{typfuchsia}{HTML}{F012BE}
	\definecolor{typmaroon}{HTML}{85144B}
	\definecolor{typred}{HTML}{FF4136}
	\definecolor{typorange}{HTML}{FF851B}
	\definecolor{typyellow}{HTML}{FFDC00}
	\definecolor{typolive}{HTML}{3D9970}
	\definecolor{typgreen}{HTML}{2ECC40}
	\definecolor{typlime}{HTML}{01FF70}
	\definecolor{newgreen}{HTML}{83C702}
	\definecolor{mathchaYellow}{HTML}{F8E71C}
	\definecolor{mathchaGreen}{HTML}{7ED321}
	\definecolor{mathchaPurple}{HTML}{BD10E0}
	\definecolor{mathchaBlue}{HTML}{58B1FF}
	\definecolor{mathchaRed}{HTML}{FF5B59}
	\definecolor{metablue}{HTML}{0064E0}
\newcommand{\cD}{\mathcal{D}}
\newcommand{\cG}{\mathcal{G}}
\newcommand{\cH}{\mathcal{H}}

\newcommand{\Trade}{\textsf{Trade}}
\usepackage{enumitem}
\usepackage[style=alphabetic,natbib=true,maxcitenames=2, maxbibnames=10,backend=bibtex]{biblatex}
\allowdisplaybreaks

\usepackage{algorithm}
\usepackage{algpseudocode}
\usepackage{bbm}

\algnewcommand{\Output}[1]{%
  \Statex \hspace*{-\algorithmicindent}\hspace{-0.5mm}\textbf{Output:} #1%
}

\usepackage{cleveref}

\usepackage[dvipsnames]{xcolor}
\newcommand{\nb}[3]{{\colorbox{#2}{\bfseries\sffamily\scriptsize\textcolor{white}{#1}}}{\textcolor{#2}{\sf\small\textsf{#3}}}}

\newcommand{\anna}[1]{\nb{Anna}{ForestGreen}{#1}}
\newcommand{\albo}[1]{\nb{Albo}{red}{#1}}
\newcommand{\mat}[1]{\nb{mat}{blue}{#1}}

\usepackage[english]{babel}

\title{The Sample Complexity of Uniform Approximation for Multi-Dimensional CDFs and Fixed-Price Mechanisms}

 \author{
 Matteo Castiglioni  \quad 
Anna Lunghi \quad
Alberto Marchesi\vspace{6mm}\\
Politecnico di Milano \\\vspace{6mm}
{\textcolor{black}{\small\texttt{name.surname@polimi.it} }}}

\date{}

\begin{document}

\maketitle

\begin{abstract}

We study the sample complexity of learning a uniform approximation of an $n$-dimensional \emph{cumulative distribution function} (CDF) within an error $\epsilon > 0$, when observations are restricted to a minimal \emph{one-bit feedback}.
This serves as a counterpart to the multivariate DKW inequality under ``full feedback'', extending it to the setting of ``bandit feedback''.
Our main result shows a near-dimensional-invariance in the sample complexity: we get a uniform $\epsilon$-approximation with a sample complexity $\frac{1}{\epsilon^3}{\log\left(\frac 1 \epsilon \right)^{\mathcal{O}(n)}}$ over a arbitrary fine grid, where the dimensionality $n$ only affects logarithmic terms. As direct corollaries, we provide tight sample complexity bounds and novel regret guarantees for learning fixed-price mechanisms in small markets, such as bilateral trade settings.
\end{abstract}


\section{Introduction}

In this paper, we investigate the sample complexity of learning a uniform approximation to a multi-dimensional \emph{cumulative distribution function} (CDF) when observations are restricted to a minimal \emph{one-bit feedback}. Notably, this serves as a counterpart to the multivariate DKW inequality under ``full feedback''~\citep{massart1990tight,dvoretzky1956asymptotic}, extending it to the setting of ``bandit feedback''. Furthermore, this problem is closely related to several sample complexity and regret minimization problems that arise in economic models, such as bilateral trade (see, \emph{e.g.},~\citep{cesa2024bilateral}).
%

Formally, consider a multivariate random variable $X$ with an unknown distribution $\cD$ supported on $[0,1]^n$. We assume to interact with this distribution via oracle queries. At each round $t$, an independent sample $X_t \sim \cD$ is drawn, we select a query point $x_t \in [0,1]^n$, and we observe the single bit $\mathbb{I}[X_t \le x_t]$, where the inequality is component-wise. Our goal is to use a minimal number of rounds (queries/samples) in order to learn a function $\tilde{P}: [0,1]^n \to [0,1]$ that, with high probability, uniformly approximates the true CDF $x \mapsto \mathbb{P}_{X \sim \cD}(X \le x)$ over the entire domain $[0,1]^n$.

The one-dimensional ($n=1$) version of this problem is a classical and well-studied question with a fundamental interpretation in economics: learning a \emph{demand curve} \citep{kleinberg2003value,paseLeme}. In this setting, at each round $t$, a new buyer arrives with an unknown random valuation $V_t \in [0,1]$ for a good offered for sale. The learner (seller) posts a price $p_t \in [0,1]$ and observes only whether the sale occurs or not---that is, they observe $\mathbb{I} [ V_t \ge p_t ]$, or equivalently $\mathbb{I} [ V_t \le p_t ]$ after a simple affine transformation. Learning the demand curve thus amounts to estimating the probability that a buyer makes a purchase as a function of the proposed price, which corresponds to the CDF of the buyer’s valuation distribution.
It is well known that learning a one-dimensional CDF to within a uniform approximation error $\epsilon > 0$ requires $\tilde O(1/\epsilon^3)$ samples, and this bound is tight up to polylogarithmic factors~\citep{kleinberg2003value,paseLeme}.

A natural multi-dimensional extension is a model of \emph{small markets} involving multiple buyers and sellers. This encompasses settings such as bilateral trade \citep{cesa2021regret,cesa2024bilateral,azar2022alpha, bernasconi2024no, lunghi2025better}---featuring a single buyer and a single seller---as well as the joint-ads model studied in \citep{aggarwal2024selling, di2025nearly}, which captures scenarios where a merchant and a brand jointly purchase advertisements for a product and both benefit when the ad is displayed. More broadly, it includes markets with multiple buyers and/or sellers \citep{lunghi2025online, babaioff2024learning}. Formally, we consider small markets with $n_s$ sellers and $n_b$ buyers, whose valuations are jointly distributed according to a probability distribution $\mathcal{V}$ supported on $[0,1]^{n_s+n_b}$. The learner is an intermediary that facilitates trades by implementing \emph{fixed-price mechanisms}, represented as vectors of prices $p \in [0,1]^{n_s+n_b}$ proposed to the agents (sellers and buyers). A trade occurs only if all agents accept their proposed prices---that is, each seller’s valuation lies below their price, and each buyer’s valuation lies above it.
%
%
%
By means of a simple affine transformation, we show that learning the demand curve in small markets can be recast as the problem of uniformly approximating an appropriate $n$-dimensional CDF.
%
%
%

Given the well-known ``curse of dimensionality'' that afflicts many learning problems, one might expect the sample complexity of uniformly approximating an $n$-dimensional CDF to deteriorate rapidly as the dimension $n$ increases.
Indeed, a na\"ive grid-based approach would require $1/\epsilon^2$ samples for each point of the grid, whose size is exponential in $n$. 
Surprisingly, we show that this is \emph{not} the case: the sample complexity of the multi-dimensional problem is essentially equivalent to that of the one-dimensional setting. In particular, the dependence on the error parameter $\epsilon > 0$ remains $\frac 1 {\epsilon^3} \log\left(\frac 1 \epsilon \right)^{\mathcal{O}(n)} $, with the dimension $n$ only appearing in logarithmic factors.
This is markedly better than classical results for Lipschitz functions, whose bounds depend strongly on $n$ (see, \emph{e.g.}, \citep{kleinberg2019bandits,bubeck2008online,slivkins2011multi}). Our result crucially leverages a key structural property of CDFs---namely, their inherent ``sparsity'', arising from the fact that the cumulative probability mass sums to one.

A key challenge is that, for general distributions, uniformly approximating a CDF over the entire domain $[0,1]^n$ is impossible: an infinitesimal probability mass (\emph{e.g.}, a Dirac delta) can be ``hidden'' and thus cannot be detected with any reasonable number of queries. Therefore, we need to either make some assumptions on the structure of the distributions or relax the uniform approximation requirement. One of our main results is on the sample complexity of uniform approximation for CDFs of distributions that admit \emph{bounded density}, and it is formally stated in the following.
%
%
\begin{theorem-non}[\Cref{theo: mainSmooth}]
    Assume that the distribution $\cD$ admits a probability density function upper bounded by $\sigma > 0$. Then, there exists an algorithm that, given an accuracy $\epsilon > 0$ and a confidence $\delta \in (0,1)$, uses $\frac{1}{\epsilon^3}\log(\sigma n/(\epsilon\delta))^{\mathcal{O}(n)}$ queries and outputs a function $\widetilde P : [0,1]^n \to [0,1]$ such that, with probability at least $1-\delta$, the following holds:
      \[\left|\mathbb{P}_{x \sim \cD}(X\le x)-\widetilde{P}(x)\right|\le \epsilon \quad \forall x  \in [0,1]^n.\]
\end{theorem-non}

This result shows that, for any constant dimension $n$, the sample complexity is $\tilde{\mathcal{O}}(1/\epsilon^3)$, matching the one-dimensional lower bound (up to polylogarithmic factors). This highlights an interesting ``dimensional invariance'' in the sample complexity of learning uniform approximations to multi-dimensional CDFs.
%
%
Notice that our algorithm has to output an estimate for all $x \in [0,1]^n$. This is accomplished by returning a function (or, more formally, a polynomial-time algorithm) that can be queried to estimate the CDF at any desired point $x \in [0,1]^n$.

\paragraph{Relation with the Multivariate DKW Inequality}

A counterpart to our ``bandit feedback'' sample complexity result in the ``full feedback'' setting is the multivariate \emph{Dvoretzky–Kiefer–Wolfowitz} (DKW) inequality~\citep{naaman2021tight}. With our notation, the DKW inequality states that, given an accuracy $\epsilon > 0$, a confidence $\delta \in (0,1)$, and $T=\mathcal{O}(\log(n/\delta) + 1/\epsilon^2)$ i.i.d.~samples $X_1,\ldots,X_T$ from $\cD$, it holds:
\[ \sup_{x \in [0,1]^n} \left\lvert \frac{1}{T}\sum_{t \in [T]} \mathbb{I}[X_t\le x]-\mathbb{P}_{x \sim \cD} (X\le x) \right\rvert\le \epsilon\] with probability at least $1-\delta$. Comparing our result to the DKW inequality, we observe (besides the additional assumptions) the classical sample complexity gap between bandit and full feedback: our bound exhibits a dependency of $\mathcal{O}(1/\epsilon^3)$ versus the full-feedback dependency of $\mathcal{O}(1/\epsilon^2)$. Furthermore, our dependence on the number of dimensions $n$ is significantly worse. This is expected, as the full-feedback setting provides an $n$-dimensional feedback (the sample $X_t$), which is substantially richer than our (dimension-independent) one-bit feedback.

\subsection{Overview of the Results}

Most of the paper is devoted to deriving a result that does not rely on the bounded density assumption. In particular, we focus on relaxing the uniform approximation requirement to hold only over a finite uniform grid of $[0,1]^n$, while allowing for any probability distribution $\mathcal{D}$, without assuming bounded density. As we discuss below, this result not only provides a proof of the theorem for distributions with bounded density but also has implications for learning problems arising in small markets.
%

\paragraph{The Sample Complexity of Multi-Dimensional CDFs Over a Grid}

As noted earlier, approximating a CDF uniformly over $[0,1]^n$ is impossible for general distributions. However, we show that it is possible to learn an approximation over an arbitrarily fine-grained grid. We consider uniform grids $\mathcal{G}_K$ over $[0,1]^n$ with resolution $1/K$ ($K \in \mathbb{N}$)---that is, uniform grids of $K^n$ points.

\begin{theorem-non}[\Cref{theo: main}]
    There exists an algorithm that, given an accuracy $\epsilon > 0$, a confidence $\delta \in (0,1)$, and a uniform grid $\mathcal{G}_K$ over $[0,1]^n$, 
    uses $\frac{1}{\epsilon^3} \log(K/\delta)^{\mathcal{O}(n)}$ samples and outputs a function $\widetilde P: \mathcal{G}_K \to [0,1]$ such that, with probability at least $1-\delta$, the following holds:
      \[\left|\mathbb{P}_{X \sim \cD}(X \le x )-\widetilde{P}(x)\right|\le \epsilon \quad \forall x \in \mathcal{G}_K.\]
\end{theorem-non}

One crucial property of our result is that the dependence on the inverse grid resolution \( K \) is only logarithmic.
This allows us to set $K$ to be polynomially large (\emph{e.g.}, $K \approx 1/\epsilon$) while only paying a polylogarithmic factor in the sample complexity. This is fundamental, since otherwise we cannot use a sufficiently fine-grained grid and still get a sample complexity of the order of $\frac{1}{\epsilon^3} \log(K/\delta)^{\mathcal{O}(n)}$. This has powerful consequences for a broad class of learning problems concerned with fixed-price mechanisms in small markets.
Notice that our result for distributions with bounded density directly follows by observing that, for sufficiently large $K$, the uniform grid $\mathcal{G}_K$ provides an accurate approximation to the CDF over the entire domain $[0,1]^n$.
%

\paragraph{The Sample Complexity of Fixed-Price Mechanisms}
%
%
A direct consequence of our results is a bound on the sample complexity of learning an approximately-optimal fixed-price mechanism in small markets. In such settings, the intermediary usually aims at finding some $p \in [0,1]^n$ that maximizes an objective of the form $ \mathbb{E}_{V \sim \mathcal{V}}[\Trade(V,p)] f(p)$, where $\Trade(V,p)$ is a binary value equal to one only if the trade occurs under valuations $V$ and prices $p$, while $f : [0,1]^n \to \mathbb{R}$ is a \emph{known} function encoding the intermediary's utility for a successful trade. In most of the economic settings of interest, the function $f$ is $L$-Lipschitz.\footnote{Indeed, we only require the function $f$ to be one-sided Lipschitz, \emph{i.e.}, Lipschitz only along one specific direction.}
For instance, the most common example of $f$ is the \emph{revenue} of the intermediary, simply defined as the difference between the sum of buyers' prices and that of sellers' ones.
%
%
Under the Lipschitzness assumption, it suffices to approximate the objective over a sufficiently fine-grained uniform grid over $[0,1]^n$, and then extend the approximation guarantees to the entire domain. Clearly, to achieve an approximation error of at most $\epsilon > 0$, the grid resolution $K$ must be on the order of $1/\epsilon$. Hence, our polylogarithmic dependence on the resolution of the grid plays a crucial role in providing optimal bounds. 
%
%
%
%
%
These observations allow to prove the following:
\begin{theorem-non}[\Cref{thm:pricing}]
    There exists an algorithm that, given an error $\epsilon>0$, a confidence $\delta \in (0,1)$, and a $L$-Lipschitz function $f: [0,1]^n \to [0,1]$, uses $\frac{1}{\epsilon^3}\log(nL/(\delta\epsilon))^{\mathcal{O}(n)}$ samples and outputs a fixed-price mechanism $p^\star \in [0,1]^n$ such that, with probability at least $1-\delta$, the following holds:\textnormal{
      \[ \mathbb{E}_{V \sim \mathcal{V}}[\Trade(V,p^\star)] f(p^\star) \geq \max_{p \in [0,1]^n} \mathbb{E}_{V \sim \mathcal{V}}[\Trade(V,p)] f(p)- \epsilon.\]}
\end{theorem-non}
%
%

%
Notably, this result gives a tight bound on the sample complexity for the problem of learning a revenue-maximizing fixed-price mechanism in our (multi-dimensional) small market model. Indeed, for any constant $n$, the $\tilde{\mathcal{O}}(1/\epsilon^3)$ sample complexity upper bound matches the lower bound for the (single-dimensional) single-buyer pricing problem~\citep{kleinberg2003value}.
%

\paragraph{Regret Minimization for Fixed-Price Mechanisms}
Our sample complexity result for learning approximately-optimal fixed-price mechanisms in small markets can also be translated into an online learning framework, where the goal is to maximize the performance during $T \in \mathbb{N}$ rounds of learning. Specifically, in such a setting, the goal is to minimize the \emph{regret} with respect to always choosing a fixed-price mechanism that is \emph{optimal in hindsight}. By applying our algorithm for the sample complexity problem within a standard explore-then-commit scheme, we can upper bound the total regret. This leads to a regret bound of the order of $\log(nT)^{\mathcal{O}(n)} T^{3/4}$, as stated in the following:
\begin{theorem-non}[\Cref{theo: mainFixedPriceRegret}]
    There exists an algorithm that, given a number of rounds $T \in \mathbb{N}$ and a $L$-Lipschitz function $f: [0,1]^n \to [0,1]$, achieves regret:\textnormal{
      \[ \textsf{REG}_T \coloneqq T \cdot \max_{p \in [0,1]^n} \mathbb{E}_{V \sim \mathcal{V}}[\Trade(V,p)] f(p) - \sum_{t\in [T]} \mathbb{E}_{V \sim \mathcal{V}}[\Trade(V,p_t)]  f(p_t) \le \log(nLT)^{\mathcal{O}(n)} T^{3/4}.\]}
\end{theorem-non}
%
%
%

Unfortunately, in this setting our result is \emph{not} tight. Indeed, for the (one-dimensional) single-buyer pricing problem, it is well known that $\tilde{\mathcal{O}}(T^{2/3})$ regret is achievable~\citep{kleinberg2003value}. However, our bound is already non-trivial in two-dimensional settings. A na\"ive approach consisting in applying a standard bandit algorithm over an $n$-dimensional grid would typically yield regret that scales poorly with $n$, whereas in our bounds $n$ only affects logarithmic terms. Our algorithm matches the na\"ive approach for $n=2$, and it already outperforms it for settings in which $n \geq 3$.

\subsection{Challenges and Techniques}


In this paper, we primarily focus on learning a uniform approximation to a CDF over a uniform grid of $[0,1]^n$,  without imposing any specific assumption on the underlying 
probability distribution.
As discussed earlier, obtaining guarantees that hold over the entire domain can be reduced (under the bounded density assumption) to selecting a sufficiently fine-grained grid over which the CDF needs to be estimated.  
Hence, throughout this section, we consider the problem defined over a grid.

Estimating the CDF of an $n$-dimensional random variable is challenging. Indeed, a na\"ive grid-based approach would require $\tilde{\mathcal{O}}(K^n / \epsilon^2)$ samples, by considering the values of the CDF at the grid points as if they were \emph{not} related to each other. For moderately large $K$, this is far worse than the sample complexity bound of $\tilde{\mathcal{O}}(1/\epsilon^3)$ that we aim to achieve in this paper.

Therefore, any improvement must come from exploiting the inherent structure of the CDF:
\begin{itemize}
    \item $\mathbb{P}_{X \sim \cD}(X \in [0,1]^n) = 1$, \emph{i.e.}, the probability over the entire domain $[0,1]^n$ sums to ones;
    \item $\mathbb{P}_{X \sim \cD}(X \le x) \le \mathbb{P}_{X \sim \cD}(X \le y)$ whenever $x \le y$ holds component-wise.
\end{itemize}

The two properties above can be exploited to estimate the CDF at a given point even if such a point has never been directly queried. Intuitively, suppose that we are given a partition of the domain $[0,1]^n$ into hyperrectangles, whose associated probabilities are known. Then, given a point $x \in [0,1]^n$, we can bound $\mathbb{P}_{X \sim \cD}(X \le x)$ between the total probability of all hyperrectangles completely contained in $\mathcal{X} \coloneq \{y \in [0,1]^n \mid y \le x\}$ and
the total probability of all hyperrectangles that are at least partially contained in $\mathcal{X}$.
The accuracy of the resulting estimate depends on the number of hyperrectangles lying on the ``frontier'' of $\mathcal{X}$ and their associated probabilities. Figure~\ref{fig: delta prob} illustrates this idea in dimension $n=2$. The gray regions determine the uncertainty of the estimate.

Moreover, since the total probability over the entire domain $[0,1]^n$ is equal to one, there can be at most $1/\epsilon$ hyperrectangles with probability at least $\epsilon > 0$. This immediately provides an intuitive upper bound on the number of the relevant hyperrectangles. However, having a single partition of $1/\epsilon$ hyperrectangles with probability $\epsilon$ is \emph{not} enough to get a decent estimate of the CDF at each point of the grid. Indeed, in the worst case the number of hyperrectangles on the ``frontier'' could be of the order of $1/\epsilon$, meaning that no actual guarantees on the estimate can be provided.

Our main idea is to design an algorithm that, by using an adaptively-constructed  representative family of hyperrectangles with a number of elements of the order of $\widetilde{\mathcal{O}}(1/\epsilon)$, is able to offer the desired approximation guarantees at each point of the grid.

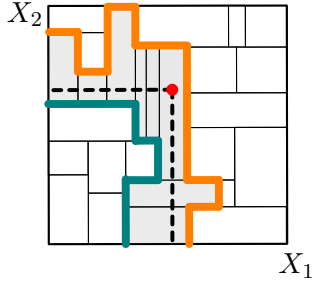
\begin{figure}
    \centering
\begin{tikzpicture}[x=0.75pt,y=0.75pt,yscale=-0.5,xscale=0.5,
    every path/.style={line cap=round,line join=round}]

\draw[thick] (174.77,265.15) -- (174.35,24.65) -- (414.85,24.23) -- (415.27,264.73) -- cycle ;

\draw   (174.77,265.15) -- (174.65,195.15) -- (214.65,195.08) -- (214.77,265.08) -- cycle ;
\draw   (214.77,265.08) -- (214.68,213.58) -- (252.37,213.51) -- (252.46,265.01) -- cycle ;
\draw   (214.68,213.58) -- (214.59,161.17) -- (252.28,161.11) -- (252.37,213.51) -- cycle ;

\draw[fill={rgb,255:red,155; green,155; blue,155}, fill opacity=0.2]
  (252.46,265.01) -- (252.39,227.51) -- (316.39,227.4) -- (316.46,264.9) -- cycle ;

\draw   (174.28,161.15) -- (174.21,123.65) -- (262.21,123.5) -- (262.28,161) -- cycle ;
\draw   (174.34,195.15) -- (174.28,161.15) -- (214.28,161.08) -- (214.34,195.08) -- cycle ;
\draw   (252.34,198.51) -- (252.28,161.01) -- (286.28,160.96) -- (286.34,198.46) -- cycle ;

\draw[fill={rgb,255:red,155; green,155; blue,155}, fill opacity=0.2]
  (252.39,227.51) -- (252.35,200.51) -- (346.35,200.35) -- (346.39,227.35) -- cycle ;

\draw[fill={rgb,255:red,155; green,155; blue,155}, fill opacity=0.2]
  (262.28,161) -- (262.11,66.22) -- (273.11,66.2) -- (273.28,160.98) -- cycle ;
\draw[fill={rgb,255:red,155; green,155; blue,155}, fill opacity=0.2]
  (174.21,123.65) -- (174.09,50.65) -- (204.09,50.6) -- (204.21,123.6) -- cycle ;
\draw[fill={rgb,255:red,155; green,155; blue,155}, fill opacity=0.2]
  (204.21,123.6) -- (204.16,91.6) -- (233.16,91.55) -- (233.21,123.55) -- cycle ;
\draw[fill={rgb,255:red,155; green,155; blue,155}, fill opacity=0.2]
  (233.21,123.55) -- (233.04,24.55) -- (262.04,24.5) -- (262.21,123.5) -- cycle ;

\draw   (174.4,51.65) -- (174.35,24.65) -- (232.04,24.55) -- (232.09,51.55) -- cycle ;
\draw   (204.16,91.6) -- (204.09,51.6) -- (234.09,51.55) -- (234.16,91.55) -- cycle ;
\draw   (262.11,66.22) -- (262.04,24.22) -- (356.04,24.06) -- (356.11,66.06) -- cycle ;


\draw   (314.35,200.41) -- (314.25,147.41) -- (362.25,147.32) -- (362.35,200.32) -- cycle ;
\draw   (314.19,112.41) -- (314.11,66.19) -- (364.11,66.1) -- (364.19,112.32) -- cycle ;
\draw   (316.46,264.9) -- (316.39,227.4) -- (414.39,227.23) -- (414.46,264.73) -- cycle ;
\draw   (314.25,147.41) -- (314.19,112.41) -- (414.68,112.23) -- (414.75,147.23) -- cycle ;
\draw   (346.39,227.35) -- (346.35,200.35) -- (362.35,200.32) -- (362.39,227.32) -- cycle ;
\draw   (362.39,227.32) -- (362.25,147.19) -- (415.25,147.09) -- (415.39,227.23) -- cycle ;
\draw   (356.11,66.06) -- (356.04,24.06) -- (373.04,24.03) -- (373.11,66.03) -- cycle ;
\draw   (364.19,112.32) -- (364.11,66.18) -- (415.11,66.1) -- (415.19,112.23) -- cycle ;
\draw   (373.11,66.18) -- (373.04,24.3) -- (415.04,24.23) -- (415.11,66.1) -- cycle ;

\fill[gray!30, opacity=0.6] (286.11,66.13) rectangle (314.35,200.46);
\draw[color=black] (286.11,66.13) rectangle (314.35,200.46);

\fill[gray!30, opacity=0.6] (273.11,66.13) rectangle (285.35,160.46);
\draw[color=black] (273.11,66.13) rectangle (286.35,160.46);


\draw[black, dashed, line width=1.5pt]
    (299.34,109.18) -- (299.46,263.43); 
\draw[black, dashed, line width=1.5pt]
    (299.34,109.18) -- (173.19,109.65); 

\draw[color=teal, line width=3]
  (252.46,265.01) -- (252.35,199.51);
\draw[color=teal, line width=3]
  (262.28,161) -- (262.21,123.5);
\draw[color=teal, line width=3]
  (174.21,123.65) -- (262.21,123.5);
\draw[color=teal, line width=3]
  (262.21,160.65) -- (285,160.5);
\draw[color=teal, line width=3]
  (285,160.65) -- (285,200.5);  
  \draw[color=teal, line width=3]
  (285,200.5) -- (255,200.5);

\draw[color=orange , line width=3]
  (316.46,264.9) -- (316.39,227.4);
\draw[color=orange, line width=3]
  (316.39,227.4) -- (346.39,227.35);
\draw[color=orange, line width=3]
  (346.39,227.35) -- (346.35,200.35);
\draw[color=orange, line width=3]
  (314.35,200.41) -- (346.35,200.35);
\draw[color=orange, line width=3]
  (314.35,200.41) -- (314.11,65.41);
\draw[color=orange, line width=3]
  (262.11,64.5) -- (262.04,25.5);
\draw[color=orange, line width=3]
  (262.11,64.5) -- (314.39,64.5);
\draw[color=orange, line width=3]
  (234.11,25.5) -- (262.11,25.5);
\draw[color=orange, line width=3]
  (234.11,25.5) -- (234.11,91);
  \draw[color=orange, line width=3]
  (234.11,91) -- (204.11,91);
  \draw[color=orange, line width=3]
  (204.11,91) -- (204.11,51);
  \draw[color=orange, line width=3]
  (204.11,51) -- (175.11,51);

\draw[line width=0.5pt, draw=red, fill=red]
  (299.34,109.18) circle (4pt);

\node[anchor=north west, inner sep=0.75pt] at (404,272) {$\displaystyle X_{1}$};
\node[anchor=north west, inner sep=0.75pt] at (130,18) {$\displaystyle X_{2}$};

\end{tikzpicture}

\caption{Illustration of the uncertainty region associated with $\mathbb{P}_{X \sim \cD}(X\le x)$.}
    \label{fig: delta prob}
\end{figure}

\paragraph{Our Approach}
At a high level, our approach is neatly divided into two steps.  
The first one consists in adaptively querying the CDF, with the goal of identifying a representative family of hyperrectangles, while also computing probability estimates for them.  
The second step uses the representative family and the estimates to assemble the CDF estimator for the grid.

At a conceptual level, our method is driven by two key principles.
\begin{itemize}
  \item \textbf{Recursive decomposition of the probability space.}  
    The $n$-dimensional unit cube $[0,1]^n$ is recursively split into hyperrectangles, one dimension at a time.  
    Each step involves ``extending'' a $j$-dimensional hyperrectangle along the $(j+1)$-th dimension so as to create $(j+1)$-dimensional hyperrectangles that either (i) carry a small probability mass (at most of the order of $\mathcal{O}(\epsilon)$) or (ii) have length at most the grid resolution $1 / K$ along the $(j+1)$-th dimension. Intuitively, this approach reduces a high-dimensional problem to a sequence of one-dimensional sub-problems.

  \item \textbf{Representative families to control error accumulation.}  
    Estimating the probability of every hyperrectangle independently would cause the total error to grow linearly in the number of hyperrectangles.  
    To avoid this, we build representative families of hyperrectangles so that the ``frontier'' of $\mathcal{X}$ described above can be written as the union of only a logarithmic (in the quantity $1 / \epsilon$) number of hyperrectangles. Consequently, the estimation error crucially accumulates only logarithmically rather than linearly.
\end{itemize}

Our approach involves five main procedures: the first four pertain to the initial step, and the final one to the subsequent step. In the following, we provide more details on each of them.
%

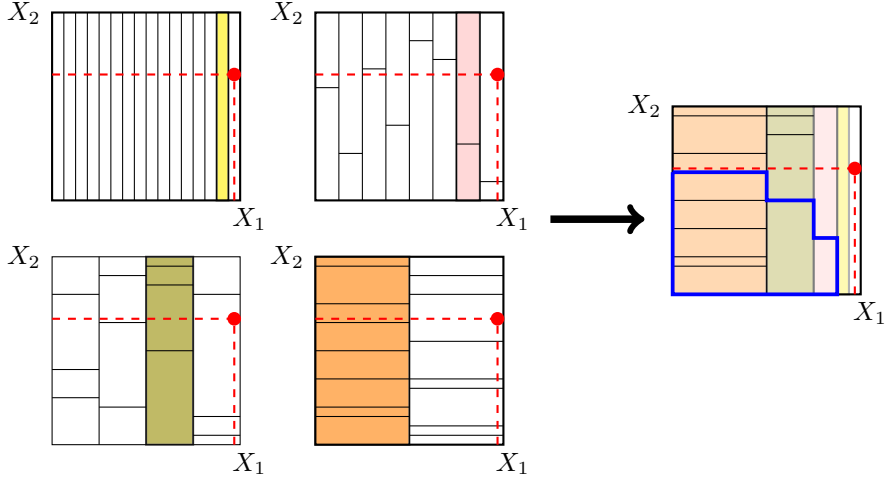
\begin{figure}[t]
\begin{center}
\begin{tikzpicture}[scale=2.5]


\begin{scope}[yscale=-1]
    
\draw[thick] (0,0) rectangle (1,1);

\draw[] (1/16,0) rectangle (1/16,1);
\draw[] (2/16,0) rectangle (2/16,1);
\draw[] (3/16,0) rectangle (3/16,1);
\draw[] (4/16,0) rectangle (4/16,1);
\draw[] (5/16,0) rectangle (5/16,1);
\draw[] (6/16,0) rectangle (6/16,1);
\draw[] (7/16,0) rectangle (7/16,1);
\draw[] (8/16,0) rectangle (8/16,1);
\draw[] (9/16,0) rectangle (9/16,1);
\draw[] (10/16,0) rectangle (10/16,1);
\draw[] (11/16,0) rectangle (11/16,1);
\draw[] (12/16,0) rectangle (12/16,1);
\draw[] (13/16,0) rectangle (13/16,1);
\draw[] (14/16,0) rectangle (14/16,1);
\draw[] (15/16,0) rectangle (15/16,1);

\draw[fill=yellow, opacity=0.6, thick] (14/16,0) rectangle (15/16,1);

\node [circle, fill=red, scale=0.5] at (15.5/16,0.33) {}  ;
\draw[thick,dashed, color=red] (0,0.33) -- (1,0.33);
 \draw[thick,dashed, color=red] (15.5/16,1) -- (15.5/16,0.33); 

\node[font=\small] at (1.05,1.1) { $X_1$} ;
\node[font=\small] at (-0.15,0) { $X_2$} ;

\end{scope}

\begin{scope}[shift={(1.4,0)},, yscale=-1]
    \draw[thick] (0,0) rectangle (1,1);
\draw[] (1/8,0) rectangle (1/8,1);
\draw[] (2/8,0) rectangle (2/8,1);
\draw[] (3/8,0) rectangle (3/8,1);
\draw[] (4/8,0) rectangle (4/8,1);
\draw[] (5/8,0) rectangle (5/8,1);
\draw[] (6/8,0) rectangle (6/8,1);
\draw[] (7/8,0) rectangle (7 /8,1);

\draw[fill=pink, opacity=0.6, thick] (6/8,0) rectangle (7/8,1);

\draw[] (0,0.4) rectangle (1/8,0.4);
\draw[] (1/8,0.75) rectangle (2/8,0.75);
\draw[] (2/8,0.3) rectangle (3/8,0.3);
\draw[] (3/8,0.6) rectangle (4/8,0.6);
\draw[] (4/8,0.15) rectangle (5/8,0.15);
\draw[] (5/8,0.25) rectangle (6/8,0.25);
\draw[] (6/8,0.7) rectangle (7/8,0.7);
\draw[] (7/8,0.9) rectangle (1,0.9);

\node [circle, fill=red, scale=0.5] at (15.5/16,0.33) {}  ;
\draw[thick,dashed, color=red] (0,0.33) -- (1,0.33);
 \draw[thick,dashed, color=red] (15.5/16,1) -- (15.5/16,0.33); 
\node[font=\small] at (1.05,1.1) { $X_1$} ;
\node[font=\small] at (-0.15,0) { $X_2$} ;
\end{scope}

\begin{scope}[shift={(0,-1.3)}, yscale=-1]
    \draw[] (0,0) rectangle (1,1);
\draw[] (1/4,0) rectangle (1/4,1);
\draw[] (2/4,0) rectangle (2/4,1);
\draw[] (3/4,0) rectangle (3/4,1);

\draw[fill=olive, opacity=0.6, thick] (3/4,0) rectangle (1/2,1);

\draw[] (0,0.2) rectangle (1/4,0.2);
\draw[] (0,0.6) rectangle (1/4,0.6);
\draw[] (0,0.75) rectangle (1/4,0.75);

\draw[] (1/4,0.1) rectangle (2/4,0.1);
\draw[] (1/4,0.35) rectangle (2/4,0.35);
\draw[] (1/4,0.8) rectangle (2/4,0.8);

\draw[] (2/4,0.05) rectangle (3/4,0.05);
\draw[] (2/4,0.15) rectangle (3/4,0.15);
\draw[] (2/4,0.5) rectangle (3/4,0.5);

\draw[] (3/4,0.95) rectangle (1,0.95);
\draw[] (3/4,0.85) rectangle (1,0.85);
\draw[] (3/4,0.2) rectangle (1,0.2);

\node [circle, fill=red, scale=0.5] at (15.5/16,0.33) {}  ;
\draw[thick,dashed, color=red] (0,0.33) -- (1,0.33);
 \draw[thick,dashed, color=red] (15.5/16,1) -- (15.5/16,0.33); 
 
 \node[font=\small] at (1.05,1.1) { $X_1$} ;
\node[font=\small] at (-0.15,0) { $X_2$} ;
\end{scope}

\begin{scope}[shift={(1.4,-1.3)}, yscale=-1]
    \draw[thick] (0,0) rectangle (1,1);
\draw[] (1/2,0) rectangle (1/2,1);

\draw[fill=orange, opacity=0.6, thick] (0,0) rectangle (1/2,1);

\draw[] (0,0.85) -- (1/2,0.85);
\draw[] (0,0.8) -- (1/2,0.8);
\draw[] (0,0.65) -- (1/2,0.65);
\draw[] (0,0.5) -- (1/2,0.5);
\draw[] (0,0.35) -- (1/2,0.35);
\draw[] (0,0.25) -- (1/2,0.25);
\draw[] (0,0.05) -- (1/2,0.05);

\draw[] (1/2,0.95) -- (1,0.95);
\draw[] (1/2,0.9) -- (1,0.9);
\draw[] (1/2,0.7) -- (1,0.7);
\draw[] (1/2,0.65) -- (1,0.65);
\draw[] (1/2,0.45) -- (1,0.45);
\draw[] (1/2,0.2) -- (1,0.2);
\draw[] (1/2,0.1) -- (1,0.1);

\node [circle, fill=red, scale=0.5] at (15.5/16,0.33) {}  ;
\draw[thick,dashed, color=red] (0,0.33) -- (1,0.33);
 \draw[thick,dashed, color=red] (15.5/16,1) -- (15.5/16,0.33);  

\node[font=\small] at (1.05,1.1) { $X_1$} ;
\node[font=\small] at (-0.15,0) { $X_2$} ;
\end{scope}

\begin{scope}[shift={(3.3,-0.5)}, yscale=-1]
    \draw[thick] (0,0) rectangle (1,1);
\draw[] (1/2,0) rectangle (1/2,1);
\draw[fill=orange, opacity=0.3, thick] (0,0) rectangle (1/2,1);
\draw[fill=olive, opacity=0.3, thick] (3/4,0) rectangle (1/2,1);
\draw[fill=pink, opacity=0.3, thick] (6/8,0) rectangle (7/8,1);
\draw[fill=yellow, opacity=0.3, thick] (14/16,0) rectangle (15/16,1);

\draw[] (0,0.85) -- (1/2,0.85);
\draw[] (0,0.8) -- (1/2,0.8);
\draw[] (0,0.65) -- (1/2,0.65);
\draw[] (0,0.5) -- (1/2,0.5);
\draw[] (0,0.35) -- (1/2,0.35);
\draw[] (0,0.25) -- (1/2,0.25);
\draw[] (0,0.05) -- (1/2,0.05);

\draw[] (2/4,0.05) rectangle (3/4,0.05);
\draw[] (2/4,0.15) rectangle (3/4,0.15);
\draw[] (2/4,0.5) rectangle (3/4,0.5);

\draw[] (6/8,0.7) rectangle (7/8,0.7);

\node [circle, fill=red, scale=0.5] at (15.5/16,0.33) {}  ;
\draw[thick,dashed, color=red] (0,0.33) -- (1,0.33);
 \draw[thick,dashed, color=red] (15.5/16,1) -- (15.5/16,0.33);


\draw[color=blue, line width=0.5mm] (0,0.35)-- (1/2,0.35)-- (1/2,0.5) -- (3/4,0.5) --(3/4,0.7)-- (7/8,0.7)-- (7/8,1)-- (0,1)-- (0,0.35);


\node[font=\small] at (1.05,1.1) { $X_1$} ;
\node[font=\small] at (-0.15,0) { $X_2$} ;
 
\end{scope}


\draw[->, line width=1mm] (2.65,-1.1) -- (3.15,-1.1);

\end{tikzpicture}
\caption{Visual representation of how different representative hyperrectangles are used to compose an estimate in two dimensions ($n=2$). 
The left side of the picture shows the partitions that define the representative family of intervals along the horizontal dimension. For each of these partitions, a further partitioning along the second dimension is constructed.
The right side of the picture illustrates the hyperrectangle composition used to estimate the CDF at the red point.}\label{fig: rep new}

\end{center}
\end{figure}

\paragraph{Estimating the Probability of a Hyperrectangle}
The first procedure needed by our method is one that estimates the probability mass of a hyperrectangle. Obtaining an estimate with accuracy $\epsilon > 0$ requires $\mathcal{O}(1/\epsilon^2)$ samples (queries), even for a single hyperrectangle. This already highlights the inherent sample complexity of our problem.

\paragraph{Binary Subdivision}
Next, we introduce a procedure that, given a $j$-dimensional hyperrectangle $A$, recursively partitions the $(j+1)$-th dimension into a set $\mathcal{I}$ of sub-intervals. Such a partition is built so that each $(j+1)$-dimensional hyperrectangle $A \times I$ for $I \in \mathcal{I}$ either (i) carries a probability at most of the order of $\mathcal{O}(\epsilon)$ or (ii) is sufficiently ``thin'' along the new added dimension, namely it in that dimension is at most the resolution of the grid $1/K$. In the former case, the probability error is negligible, while, in the latter case, the endpoints of the interval $I$ lie on the grid. This procedure is effective since the probability mass is inherently ``sparse'', producing roughly $\frac 1 \epsilon \mathbb{P}(X \in A)$ intervals.

\paragraph{Building a Representative Set of Intervals}
The third procedure constructs a representative set of intervals derived from $\mathcal{I}$. Directly estimating the CDF on $\mathcal{G}_K$ by using unions of intervals in $\mathcal{I}$ would accumulate error linearly in $|\mathcal{I}|$. To avoid this, we enlarge $\mathcal{I}$ into a representative family $\mathcal{I}^\star$, such that any  interval $[0,x]$ that can be represented as a union of intervals in $\mathcal{I}$ can also be represented as the union of only $\mathcal{O}(\log m)$ intervals in $\mathcal{I}^\star$.
Formally, given $m$ disjoint intervals $\mathcal{I}$ covering $[0,1]$, we construct $\mathcal{I}^\star$ of size at most $2m$, where each element is a union of intervals from $\mathcal{I}$, and every union of consecutive intervals in $\mathcal{I}$ can be expressed as a union of at most $\mathcal{O}(\log m)$ intervals in $\mathcal{I}^\star$.
Interestingly, $\mathcal{I}^*$ is build as the union of a logarithmic number of partitions.
Figure~\ref{fig: rep new} illustrates this construction, showing how the multiple partitions in $\mathcal{I}^\star$ can be combined to minimize the number of intervals 
whose union estimates the probability of a given $x$. 

\paragraph{Representative Hyperrectangles Identification}
We identify the hyperrectangles recursively. For any identified hyperrectangle of dimension $j$, we apply the binary search procedure to generate intervals $\mathcal{I}$, each carrying roughly $\epsilon$ probability, and construct the corresponding representative family $\mathcal{I}^\star$. Then, we proceed identifying hyperrectangles of dimension $j+1$.
This reduces the $n$-dimensional estimation problem to a sequence of one-dimensional subproblems.


\paragraph{Estimating the Cumulative Probability}
Given $x \in [0,1]^n$, we construct an estimator for $\mathbb{P}_{X \sim \cD}(X \le x)$ by using the representative family of hyperrectangles generated by the other procedures.
Our goal is to ensure that the gap between the largest collection of hyperrectangles fully contained in $\mathcal{X}$ and the smallest collection that covers this set is small (Figure~\ref{fig: delta prob}).
As discussed earlier, a uniform partition into hyperrectangles of $\epsilon$ probability is insufficient.
However, by leveraging the representative family, we guarantee that the discrepancy involves only a logarithmic number of hyperrectangles, each of mass at most $\epsilon$, thus ensuring that the constructed estimator deviates from the true CDF by at most $\tilde{\mathcal{O}}(\epsilon)$ for a fixed dimensionality $n$.

\section{Preliminaries}\label{sec:prelim}

In this section, we formally introduce the notation and all the definitions needed in the paper.

\subsection{Learning Multi-Dimensional CDFs}


Our focus in this paper is primarily on the problem of learning a uniform approximation to an $n$-dimensional CDF. We denote by $\cD$ the underlying probability distribution, which we assume to be supported on the $n$-dimensional unit hypercube $[0,1]^n$. Then, the target CDF is defined as the function that maps any point $x \in [0,1]^n$ to the probability that a multivariate random variable $X$ distributed as $\cD$ has value less than or equal to $x$ (component-wise), written as $\mathbb{P}_{X\sim \cD}(X\le x)$.
%

Our learning objective can be framed within the \emph{probably approximately correct} (PAC) learning framework. Indeed, we aim at finding, with high probability, a good approximation to the CDF over the entire continuous domain $[0,1]^n$, by using the least possible number of queries to the CDF, \emph{i.e.}, with minimal \emph{sample complexity}. Notably, we consider the most challenging case in which each query to the CDF only provides \emph{one-bit feedback}, as formally discussed in the following.
%

\paragraph{Learning Interaction Model}
The learning process unfolds over multiple rounds, where in each round we are allowed to make a query to obtain feedback about the target CDF. Formally, there is a sequence of i.i.d. samples $X_1, X_2, \ldots$ drawn from the probability distribution $\cD$, where $X_t$ is the sample realized during round $t$. Then, at each round $t$, we can adaptively---based on previous observations---choose a query point ${x}_t \in [0,1]^n$ and observe the binary one-bit feedback $\mathbb{I}[X_t \le {x}_t]$. In other words, we only learn whether the sample $X_t$ is greater or smaller than the queried point $x_t$, but \emph{not} its actual value. This constitutes the minimal form of feedback sufficient to learn CDFs.
%
%

\paragraph{Learning Goal}
Ideally, given a desired accuracy $\epsilon > 0$ and confidence $\delta \in (0,1)$, our goal would be to learn a function $\widetilde{P}: [0,1]^n \to [0,1]$ that uniformly approximates the target CDF over the entire domain $[0,1]^n$ with high probability. Formally, $\widetilde{P}$ should satisfy, with probability at least $1-\delta$:
\begin{align}\label{eq:unifLearn}
    \left| \mathbb{P}_{X \sim \cD}(X \leq x)-\widetilde P(x) \right| \le \epsilon \quad \forall x \in [0,1]^n.
\end{align}
Furthermore, we aim to achieve this guarantee using the smallest possible number of rounds (queries), which we refer to as the \emph{sample complexity of uniformly approximating} an $n$-dimensional CDF.

\paragraph{Relaxations and Assumptions}
In this paper, we primarily focus on general probability distributions $\cD$. However, in such a broad setting, certain CDFs are inherently unlearnable---for instance, when a large portion of the probability mass is concentrated at a single point (as in the case of Dirac delta distributions).
Consequently, to make the learning problem well-posed, we must either relax the uniform approximation requirement in \Cref{eq:unifLearn} or impose additional assumptions on the structure of the underlying distribution $\cD$. We explore both approaches.
\begin{enumerate}
    \item We relax the requirement in \Cref{eq:unifLearn} so that it holds only over a finite uniform grid. Given $K \in \mathbb{N}$, we denote by $\mathcal{G}_K$ the uniform grid over $[0,1]^n$ of resolution $1 / K$, where each coordinate $x_i$ of a point $x \in [0,1]^n$ takes values of the form $i/K$ for some $i \in \{0, \dots, K\}$.
    \item We assume that the probability distribution $\cD$ has \emph{bounded density}, which means that it admits a  probability density function $d: [0,1]^n \to \mathbb{R}$ that is upper bounded by some constant $\sigma > 0$ everywhere. Formally, it holds $d(x) \leq \sigma$ for all $x\in [0,1]^n$.
\end{enumerate}
In the remainder of this paper, we primarily focus on the first case (finite uniform grid), since the second one can be straightforwardly reduced to it, as shown later.

\begin{remark}
    For ease of exposition, in the rest of this paper we assume that $\log_2 K \in \mathbb{N}$. Clearly, this does \emph{not} affect the order of our bound, bit it greatly simplifies calculations.
\end{remark}
%




\paragraph{Additional Notation}
Throughout this paper, we denote by $[a] \coloneqq \{ 1, \ldots, a\} $ the set of the first $a \in \mathbb{N}_+$ natural numbers. Since our algorithms heavily work with hyperrectangles living in lower-dimensional hypercubes, we need to introduce some additional useful notation. Given $j \in [n]$, we denote by $\cD_j$ the marginalization of the distribution $\cD$ over the first $j$ dimensions. Thus, $\cD_j$ is a probability distribution supported on $[0,1]^j$ and, given a $j$-dimensional point $x \in [0,1]^j$, the probability $\mathbb{P}_{X \sim \cD_j} (X \leq x)$ is the value of its CDF at $x$. Furthermore, given $j \in [n]$, we let $\mathcal{H}^j$ be the set of all the $j$-dimensional (open or closed) hyperrectangles. Specifically, each hyperrectangle $A \in \mathcal{H}^j$ can be written as the Cartesian product of $j$ one-dimensional intervals, namely $A \coloneqq \bigtimes_{i \in [j]} {I}_i$, where each ${I}_i\subseteq [0,1]$ is a (possibly open or closed) interval in $[0,1]$. Then, given any $A \in \mathcal{H}^j$, we write $\mathbb{P}_{X \sim \cD_k}(X \in A)$ to denote the probability that the marginalization of $\cD$ over the first $j$ dimensions puts on the portion of its domain $[0,1]^j$ identified by $A$.
Since our procedures work recursively, we sometimes need to refer to $0$-dimensional quantities. We use $\mathbbm{1}$ to represent the neutral element of the Cartesian product, and we artificially let $\cH^0 \coloneq \{ \mathbbm{1} \}$. Thus, \emph{e.g.}, $\mathbbm{1}\times [a,b]=[a,b]$.
Finally, given a set of intervals $\mathcal{I}$, we let $\textsc{Extremes}(\mathcal{I})$ denote the union of the extremes of all the intervals in $\mathcal{I}$.
%
%
%
%

\subsection{Learning in Small Markets}

The problem of learning uniform approximations to multi-dimensional CDFs is also intimately related to certain learning problems arising from economic models, most notably to learning \emph{fixed-price mechanisms} in \emph{small markets} settings that potentially involve multiple sellers and buyers.

We consider small markets where $n_s$ sellers are willing to sell their goods to $n_b$ potentially interested buyers. We let $n = n_s + n_b$ be the total number of agents (sellers and buyers), and we assume that each agent is labeled with an index $i \in N \coloneqq [n]$, so that we can partition $N$ into two sets $N_s$ and $N_b$, containing sellers' and buyers' indices, respectively. Each agent $i \in N$ has a valuation $V_i \in [0,1]$ for the good---either as a seller or buyer---and the $n$-dimensional vector $V \in [0,1]^n$ of all agents' valuations is drawn from some (joint) probability distribution $\mathcal{V}$.
%

In the small market settings considered in this paper, the learner acts an intermediary that facilitates the trade involving sellers and buyers. This is done by implementing a fixed-price mechanism, which consists in a vector $p \in [0,1]^n$ of prices proposed to the agents. Each seller $i \in N_s$ is offered a selling price $p_i$ and agrees to sell their good only if $p_i$ exceeds their valuation $V_i$. Conversely, each buyer $i \in N_b$ is offered a buying price $p_i$ and decides to purchase the good only if $p_i$ is less than their valuation $V_i$. Formally, given a vector $V \in [0,1]^n$ of agents' valuations and a fixed-price mechanism $p \in [0,1]^n$, we denote by $\Trade({V}, {p})$ the binary outcome of the trade:
\begin{align*}
    \Trade(V,p) \coloneqq \begin{cases}
        1 & \text{if } \mathbb{I}[ V_{i}\ge p_{i}]=1 \ \forall i \in N_b \textnormal{ and } \mathbb{I}[ V_{i}\le p_{i}]=1 \ \forall i \in N_s\\
        0 & \text{otherwise.}
    \end{cases}
\end{align*}
Notice that a trade takes place only if every seller agrees to sell and every buyer opts to buy.
%
%

In general, the intermediary gets a utility every time a trade occurs, and such a utility depends on the chosen mechanism through a known objective function $f : [0,1]^n \to \mathbb{R}$. Specifically, given a vector $V \in [0,1]^n$ of agents' valuations, when the intermediary implements a mechanism $p \in [0,1]^n$ their resulting utility is $f(p)$ if $\Trade({V}, {p}) = 1$, while it is zero otherwise.
We assume that the function $f$ is $L$-Lipschitz with respect to the $\ell_1$-norm, as this is the case in most of the small market settings of interest. Formally, we assume that $|f({p}) - f({p}')| \le L \|{p} - {p}'\|_1$ for all $p,p' \in [0,1]^n$.\footnote{Technically, our results hold even if we require the weaker assumption that the objective function $f$ is Lipschitz along certain directions. In particular, the function $f$ should be one-sided Lipschitz along the direction of decreasing buyers' prices and increasing sellers' prices.}
The most common example of objective function for small market settings is the intermediary's \emph{revenue}, which is defined as $\textsf{Rev}(p) \coloneqq \sum_{i \in N_b} p_{i} - \sum_{i \in N_s} p_{i}$, and it is $1$-Lipschitz.

The small market settings considered in this paper include as special cases the \emph{bilateral trade} problem~\citep{cesa2024bilateral} (one buyer and one seller) and the \emph{joint-ads} model studied in~\citep{aggarwal2024selling, di2025nearly} (multiple buyers), while being at the same time much more general.

\paragraph{Relation with Learning CDFs}
We remark that the small market settings described above can be framed within the formalism of learning CDFs $x \mapsto \mathbb{P}_{X \sim \cD}(X \leq x)$, by simply defining the multivariate random variable $X$ and the query point $x \in [0,1]^n$ as follows:
%
\begin{align}\label{eq:conversion}
    x_i = \left\{ \begin{array}{ll}
         p_i & \textnormal{if } i \in N_s  \\
         1 - p_i & \textnormal{if } i \in N_b,
    \end{array} \right.\quad \textnormal{and} \quad 
    X_i = \left\{ \begin{array}{ll}
         V_i & \textnormal{if } i \in N_s  \\
         1 - V_i & \textnormal{if } i \in N_b.
    \end{array} \right.
\end{align}

\paragraph{Learning Goals}

In this paper, we consider two distinct (but related) learning problems that are typically studied in the literature: the \emph{sample complexity} and \emph{regret minimization} problems.
In both of them, we take the role of the intermediary and we assume to repeatedly interact with the small market over multiple rounds, with the goal of learning an optimal fixed-price mechanism. At each round $t$, a valuation vector $V_t \in [0,1]^n$ is independently drawn from the probability distribution $\mathcal{V}$, we choose a mechanism ${p}_t \in [0,1]^n$, and we observe $\Trade({V}_t, {p}_t)$ as feedback. In other words, we only observe the binary outcome of the trade. Notice that, when the small market setting is framed within the formalism of learning CDFs, this type of feedback is equivalent to \emph{one-bit} feedback.
%

\begin{itemize}


\item
\textbf{Sample Complexity Problem.}
This is framed within a PAC learning framework. The goal is to learn an approximately-optimal fixed-price mechanism with high probability, by using the minimum possible number of rounds (queries). Formally, given an error $\epsilon > 0$ and a confidence $\delta \in (0,1)$, the goal is to find a $p^\star \in [0,1]^n$ such that, with probability at least $1-\delta$: 
\[\mathbb{E}_{{V} \sim \mathcal{V}}[\Trade({V},{{p^\star}})] f({{p^\star}})\ge \max_{{p} \in [0,1]^n} \mathbb{E}_{{V}\sim \mathcal{V}}[\Trade({V},{p})] f({p}) -\epsilon.\]
The sample complexity problem is deeply connected with the general problem of learning a uniform approximation of an $n$-dimensional CDF. However, in this case, we do \emph{not} aim at uniformly approximating the CDF, but only to learn an approximately-optimal mechanism. 
%

\item 
\textbf{Regret Minimization Problem.} This is related to the sample complexity problem, but it focuses on maximizing performance during the learning. Here, the goal is to minimize the \emph{(cumulative) regret} with respect to always choosing a fixed-price mechanism that is \emph{optimal in hindsight}. The learning interaction is as in the previous case, but we are given a time horizon $T \in \mathbb{N}$ and the goal is to minimize the regret defined as:
\begin{align*}
   \mathsf{REG}_T \coloneq T \cdot \max_{p \in [0,1]^n} \mathbb{E}_{V \sim \mathcal{V}}[\Trade(V,p)] f(p) - \sum_{t \in [T]} \mathbb{E}_{V \sim \mathcal{V}}[\Trade(V,p_t)] f(p_t).
\end{align*}
Intuitively, the regret measures the difference between the expected utility achieved by the best-in-hindsight fixed-price mechanism (\emph{i.e.}, a mechanism that knows the distribution $\mathcal{V}$) and the expected utility that we actually achieve during the learning process. Ideally, we would like to attain sub-linear regret---that is, a regret $\mathsf{REG}_T = o(T)$.

\end{itemize}

\section{Main Results and Proof Plan}


We prove our main result in two steps. In the first one (\Cref{sec:building}), we provide an algorithm that adaptively queries points in order to build a suitable \emph{representative family of hyperrectangles} with their associated probability estimates.
%
%
In the second step (\Cref{sec:estimate}), we design an algorithm that outputs an estimate of the target CDF at any point $x \in \cG_K$, given a representative family of hyperrectangles and their estimates computed by means of the algorithm in the first step of the proof. The estimation algorithm operates by decomposing the target probability $\mathbb{P}_{X \sim \cD}(X\le x)$ as the sum of the estimates associated with a suitable subset of hyperrectangles within the family.
%

The two steps above allow to prove the following main result.
\begin{theorem}\label{theo: main}
    There exists an algorithm that, given an accuracy $\epsilon > 0$, a confidence $\delta \in (0,1)$, and a uniform grid $\mathcal{G}_K$ over $[0,1]^n$ of resolution $1 / K$ (with $K \in \mathbb{N}$), uses $\frac{1}{\epsilon^3}\log(K/\delta)^{\mathcal{O}(n)}$ queries and outputs a function $\widetilde P: \mathcal{G}_K \to [0,1]$ such that, with probability at least $1-\delta$, the following holds:
      \[\left|\mathbb{P}_{X \sim \cD}(X \le x )-\widetilde{P}(x)\right|\le \epsilon \quad \forall x \in \mathcal{G}_K.\]
\end{theorem}

Notice that, if a probability distribution admits a bounded density, then a sufficiently fine-grained uniform grid can be employed to approximate its CDF well over the entire domain $[0,1]^n$. Thus, we readily get the following corollary about the sample complexity of uniform approximation for $n$-dimensional CDFs of distributions with bounded density.
\begin{corollary}
\label{theo: mainSmooth}
    Assume that the distribution $\cD$ admits density bounded by $\sigma > 0$. Then, there exists an algorithm that, given an accuracy $\epsilon > 0$ and a confidence $\delta \in (0,1)$, uses $\frac{1}{\epsilon^3}\log(\sigma n/(\epsilon\delta))^{\mathcal{O}(n)}$ samples and outputs a function $\widetilde P : [0,1]^n \to [0,1]$ such that, with probability at least $1-\delta$:
      \[\left|\mathbb{P}_{x \sim \cD}(X\le x)-\widetilde{P}(x)\right|\le \epsilon \quad \forall x  \in [0,1]^n.\]
\end{corollary}

\begin{proof}
To prove the corollary, it is sufficient to consider the algorithm in \cref{theo: main} with inputs: an accuracy $\epsilon' \coloneq \epsilon/2$, a confidence $\delta$, and a uniform grid $\cG_K$ with $K \coloneq \frac{2\sigma n}{\epsilon }$. Let $\widetilde P$ be the function returned by \cref{theo: main}. 
Then, we extend the function over the entire hypercube $[0,1]^n$ by setting $\widetilde P(x)=\widetilde P(\bar x)$ for every $x \notin \cG_K$, where $\bar x \in \cG_K$ is the closest point on the grid $\cG_K$ such that $\bar x \ge x$.
Notice that $\lVert\bar x-x\rVert_\infty\le \frac 1 K$ and, thus, $\mathbb{P}_{X\sim \cD}(x\le X \le\bar x)\le \sigma \frac{n}{K}$ by the bounded density assumption.

Then, we get the following:
\begin{align*}
\left|\mathbb{P}_{X \sim \cD}(X \le x )-\widetilde{P}(x)\right|& =  \left|\mathbb{P}_{X \sim \cD}(X \le x )-\widetilde{P}(\bar x)\right|\\
& \le \left|\mathbb{P}_{X \sim \cD}(X \le \bar x )-\widetilde{P}(\bar x)\right| + \sigma \frac{n}{K}\\
& \le \epsilon' + \sigma \frac{n}{K}\\
&\le \epsilon.
\end{align*}
This concludes the proof.
\end{proof}

\section{Hyperrectangles Identification and Probability Estimation} \label{sec:building}


This section is devoted to the first step of our proof: the design of an algorithm that constructs a representative family of hyperrectangles, 
with their corresponding probability estimates. This algorithm involves several procedures, and we divide its exposition into four parts.

\begin{itemize}
    \item In Section~\ref{sec: MCE}, we introduce \textsc{Monte Carlo Estimation (MCE)}. 
This procedure takes as inputs a hyperrectangle $A \in \mathcal{H}^j$ of dimension $j \in [n-1]$ and a point $w \in [0,1]$. Then, by means of a suitable sequence of queries to the CDF at the vertices of the hyperrectangle $A$, the procedure computes an unbiased estimate of the probability
\[
\mathbb{P}_{X \sim \mathcal{D}_{j+1}}\big(X \in A \times [0,w]\big).
\]
The estimate is built by suitably combining the one-bit feedback observed at the vertices of the hyperrectangle. A standard application of Hoeffding's inequality then shows that it is possible to achieve an estimation precision of order \( \mathcal{O}(\epsilon) \) by using a total of \(\mathcal{O}(1 / \epsilon^2)\) queries.

\item In Section~\ref{sec: bins}, we introduce \textsc{Binary Subdivision (BinS)}.  
This is a procedure that, given a hyperrectangle 
\( A \in \mathcal{H}^{j} \) of dimension $j \in [n-1]$, partitions the $(j+1)$-th dimension---the interval $[0,1]$---into a set \( \mathcal{I} \) of sub-intervals.
The partition \( \mathcal{I} \) is built in such a way that each resulting $(j+1)$-dimensional hyperrectangle \( A \times I \) with \( I \in \mathcal{I} \) either (i) carries probability at most \( \mathcal{O}(\epsilon ) \)---that is, $\mathbb{P}_{X \sim \cD_{j+1}} (X \in A \times I) \leq \mathcal{O}(\epsilon)$---or (ii) has length along the $(j+1)$-th dimension no larger than the resolution \( 1/K \) of the grid. 
The procedure works by recursively halving the interval $[0,1]$ by means of a binary search approach, and it employs the \textsc{MCE} procedure to estimate the probabilities of hyperrectangles.

\item In Section~\ref{sec: repre family}, we introduce the notion of a \emph{representative family of intervals}.  
We show that, starting from a partition \( \mathcal{I} \) of $[0,1]$ made by $m$ sub-intervals, it is possible to construct a representative family of intervals \( \mathcal{I}^\star \) of cardinality at most $2m$, so that any interval of the form \( [0,x] \) that can be expressed as a union of consecutive intervals in \( \mathcal{I} \) can also be represented as a union of at most \( \log_2 m \) disjoint intervals in \( \mathcal{I}^\star \). The family $\mathcal{I}^\star$ is built by the union of $\log_2 m$ suitably-defined partitions of $[0,1]$, obtained by merging consecutive intervals in \( \mathcal{I} \).

\item In Section~\ref{sec: fine fase 1}, we introduce \textsc{Representative Hyperrectangles Identification (RHI)}. This procedure leverages the previous components to construct a representative family of hyperrectangles with their associated probability estimates. Intuitively, the hyperrectangles in such a family are constructed one dimension at a time, so that the hyperrectangles of dimension $j+1$ are ``offspring'' of those of dimension $j$, constructed by means of the representative family of intervals originating from the partition along the $(j+1)$-th dimension computed by \textsc{BinS}.

\end{itemize}


\textsc{RHI} is the main procedure executed by the algorithm in the first step of our proof. This procedure computes all the elements described in the following. 
%
%
\begin{itemize}
    \item For each dimension $j \in [n]$, a set $\mathcal{R}^j \subseteq \cH^j$ of hyperrectangles of dimension $j$. These sets constitute the representative family of hyperrectangles.
    \item For each hyperrectangle $A \in \mathcal{R}^j$ (with $j \in [n-1]$), a partition of the $(j+1)$-th dimension---the interval $[0,1]$---into a set $\mathcal{I}(A)$ of sub-intervals constructed by means of the \textsc{BinS} procedure. 
    \item For each hyperrectangle $A\in \mathcal{R}^j$ (with $j \in [n-1]$), a representative family of intervals, called $\mathcal{I}^\star(A)$ and made by $\log_2 m$ partitions constructed starting from $\mathcal{I}(A)$.
    \item For each  hyperrectangle $A\in \mathcal{R}^j$ (with $j \in [n-1]$) and point $w \in \textsc{Extremes}(\mathcal{I}(A))$, an estimate $\widehat P(A,w)$ of the probability $\mathbb{P}_{X \sim \cD_{j+1}} (X \in A \times [0,w])$, computed by means of the \textsc{MCE} procedure. 
\end{itemize}
In order to simplify exposition, we assume that all the elements computed by \textsc{RHI} are ``global variables'', and they can be accessed and modified by all the sub-procedures invoked by \textsc{RHI}.   
%
%

\subsection{Estimation Procedure} \label{sec: MCE}


First, we describe the \textsc{Monte Carlo Estimation (MCE)} procedure, the building block of our algorithm that estimates the probability associated with hyperrectangles. In order to be better integrated with the other procedures, \textsc{MCE} estimates the probability of a hyperrectangle of dimension $j+1$ by separately considering the first $j$ dimensions and the $(j+1)$-th one. In particular, the \textsc{MCE} procedure takes as inputs a hyperrectangle $A \in \mathcal{H}^{j}$ of dimension $j \in [n-1]$ and a threshold $w \in [0,1]$ along the $(j+1)$-th dimension, where $A \coloneq \bigtimes_{i \in [j]} I_i$ for some intervals $I_i \coloneq (a_i,b_i]$ with $a_i < b_i$.
%
%
Then, \textsc{MCE} computes an estimate of the probability $\mathbb{P}_{{X}\sim \cD_{j+1}}(X \in A \times [0,w])$ and stores it in the global variable $\widehat{P}(A,w)$, in order make it readily available to the other procedures of the algorithm.

The pseudo-code for the \textsc{MCE} procedure is presented in Algorithm~\ref{alg: mce}. The procedure operates by performing \( \lceil 1 / \epsilon^2 \rceil \) queries. Each query point \(x_\tau \in [0,1]^n\) is generated by uniformly sampling a vertex of the hyperrectangle \(A\). Specifically, the first \(j\) components of \(x_\tau\) are set to the coordinates of the sampled vertex, the \((j+1)\)-th component is deterministically assigned the value \(w\), and all remaining components are fixed to one. Then, the unbiased estimate is computed by combining the observed one-bit feedback by means of the inclusion-exclusion principle.

%
\begin{remark}
    For ease of presentation, the \textsc{MCE} procedure in Algorithm~\ref{alg: mce} is written assuming that all the intervals $I_i$ are of the form $(a_i,b_i]$ with $a_i < b_i$. However, \textsc{MCE} may also be invoked by other procedures on (degenerate) intervals containing zero as their sole element---that is, $I_i \coloneq \{0\}$.
    Algorithm~\ref{alg: mce} can be straightforwardly adapted to handle this case by modifying the construction of $x_\tau$ so that $x_{\tau,i}=0$ for all $i \in [j]$ such that $I_i = \{0\}$, while leaving the remaining components unchanged. In addition, Line 6 of Algorithm~\ref{alg: mce} should be modified by replacing $2^j$ with two raised to the number of non-degenerate intervals.
    %
    %
    Moreover, the \textsc{MCE} procedure may also be invoked for $\mathbbm{1} \in \mathcal{H}^0$ (\emph{i.e.}, the neutral element of the Cartesian product in the fictitious set of $0$-dimensional hyperrectangles). In order to handle this case, Algorithm~\ref{alg: mce} can be modified by avoiding the sampling step in Line~\ref{line:sample} and setting $y_\tau \in [0,1]^n$ to be a vector of all zeros.
\end{remark}

\begin{algorithm}[!t]
\caption{\textsc{Monte Carlo Estimation ({MCE})}}\label{alg: mce}
\begin{algorithmic}[1]
    \Require Hyperrectangle $A \in \cH^j$ of dimension $j \in [n-1]$ such that $A \coloneq \bigtimes_{i \in [j]} I_i$ for some intervals $I_i \coloneq (a_i,b_i]$ with $a_i < b_i$; threshold $w \in [0,1]$; accuracy $\epsilon > 0$
    %
    \Ensure The value of $\widehat P(A,w)$ is an estimate of $\mathbb{P}_{X\sim \cD_{j+1}}(X \in A \times [0,w])$
    %
    \Statex \hrulefill
    \Function{MCE}{$A,w,\epsilon$}
    \State $N \gets \lceil 1/\epsilon^2 \rceil$ 
    \For{$\tau = 1, \ldots, N$}
        \State Sample $y_\tau $ uniformly at random from the set $\{0,1\}^j$\label{line:sample}
        \State Let $x_\tau \in [0,1]^n: x_{\tau,i} = \left\{ \begin{array}{ll}
             a_i & \text{if } y_{\tau,i}=1 \\
             b_i & \text{if } y_{\tau,i}=0 
        \end{array}\right.$ for $i\in [j]$, $x_{\tau,j+1}=w$, and $x_{\tau,i}=1$ for $i > j+1$
        \State Query $x_\tau$ and observe feedback $ \mathbb{I}[X_{\tau}\le x_{\tau}]$
        \State Compute $z_\tau$ by using the inclusion–exclusion principle: $z_\tau \gets 2^j \cdot (-1)^{\lVert y_\tau \rVert_1} \cdot \mathbb{I}[X_{\tau}\le x_{\tau}]$ 
    \EndFor
    \State $\widehat P(A,w)\gets \frac{1}{N}\sum_{\tau \in [N]} z_\tau$
    \EndFunction
\end{algorithmic}
\end{algorithm}

The following lemma shows the correctness of the \textsc{MCE} procedure. Its proof is straightforward. Indeed, once one notices that we are building an unbiased estimator of the probability associated with the hyperrectangle, the result simply follows from an application of Hoeffding's inequality.

\begin{restatable}{lemma}{lemmaEstimate}\label{theo: alg MCE}
    Let $\delta \in (0,1)$ be a confidence parameter. Then, Algorithm~\ref{alg: mce} guarantees that, with probability at least $1-\delta$, the following property holds:
  \[\left\lvert \widehat{P}(A,w) - \mathbb{P}_{X \sim \cD_{j+1}}\left(X \in A\times [0,w]\right)  \right\rvert\le \epsilon \, 2^{n-1} \;\sqrt{\frac{\ln\left(\nicefrac{2}{\delta}\right)}{2}}.\] 
    Moreover, during its execution Algorithm~\ref{alg: mce} employs $1 / \epsilon^2$ queries.\footnote{For readability, throughout the paper we omit the ceiling operator in sample complexity expressions. Naturally, the actual number of queries employed by the algorithm is an integer.}
\end{restatable}

\begin{algorithm}[!t]
\caption{\textsc{Binary Subdivision (\textsc{BinS})}}
\label{alg:bins}
\begin{algorithmic}[1]
    \Require Hyperrectangle $A \in \mathcal{H}^j$ of dimension $j\in [n-1] \cup \{0\}$; accuracy $\epsilon > 0$; confidence $\delta \in (0,1)$
    \Ensure The properties stated in \Cref{thm:binS} are satisfied
    \Statex \hrulefill
    \Function{\textsc{BinS}}{$A$, $\epsilon$, $\delta$}
        \State $\mathcal{I} \gets \varnothing$
        \State $\widehat{P}(A,0) \gets \textsc{MCE}(A,0,\epsilon)$; $\widehat{P}(A,1) \gets \textsc{MCE}(A,1,\epsilon)$ \Comment{Compute estimates at the endpoints}
        \State $\mathcal{I} \gets $\Call{\textsc{BinS-Rec}}{$A, 0, 1, \epsilon, \delta$}
        \State \Return $\mathcal{I}\cup \{0\}$ 
    \EndFunction
    \Statex \hrulefill
    \Function{\textsc{BinS-Rec}}{$A, w^1, w^2, \epsilon, \delta$}
        \If{$\widehat{P}(A,w^2)-\widehat{P}(A,w^1)-2^{n} \epsilon \sqrt{\frac{\ln(\nicefrac{4K}{\delta})}{2}} < \epsilon \; \vee \; w^2 - w^1 \leq \frac 1 K$}
            \State \Return $(w^1,w^2]$
        \EndIf
        \State $w^3 \gets \frac{w^1+w^2}{2}$ \Comment{Compute mid-point of the interval}
        \State $\widehat{P}(A,w^3) \gets \textsc{MCE}(A,w^3,\epsilon)$  \label{line:call_mce}\Comment{Estimate probability of $A \times [0,w^3]$}
        \State \Return \Call{\textsf{BinS-Rec}}{$A, w^1, w^3, \epsilon, \delta$} $\cup$ \Call{\textsf{BinS-Rec}}{$A, w^3, w^2, \epsilon, \delta$} \Comment{Recursive calls}
    \EndFunction
\end{algorithmic}
\end{algorithm}

\subsection{Building Partitions}\label{sec: bins}


In this section, we introduce the \textsc{Binary Subdivision ({BinS}}) procedure, the building block of our algorithm used to build $(j+1)$-dimensional hyperrectangles from $j$-dimensional ones, by suitably partitioning the $(j+1)$-th dimension. In particular, the \textsc{BinS} procedure takes as input an accuracy $\epsilon > 0$, a confidence $\delta \in (0,1)$, and a hyperrectangle $A\in \mathcal{H}^j$ of dimension $j \in [n-1] \cup \{0\}$, and it produces a set $\mathcal{I}$ of one-dimensional intervals that partition $[0,1]$ along the $(j+1)$-th dimension. Intuitively, the goal of the \textsc{BinS} procedure is to produce a partition $\mathcal{I}$ such that the resulting $(j+1)$-dimensional hyperrectangles defined as $A \times I$ with $I \in \mathcal{I}$ are either (i) associated with a probability mass of at most $\mathcal{O}(\epsilon)$ or (ii) very ``thin'' along the $(j+1)$-th dimension---that is, the interval $I$ has length at most $ 1 / K$. Notice that point (ii) above is needed to be able to deal with probability distributions that do \emph{not} admit a bounded density.
%
%

The pseudo-code of the \textsc{BinS} procedure is presented in Algorithm~\ref{alg:bins}. Intuitively, the procedure creates a partition $\mathcal{I}$ by means of a binary search that recursively halves the interval $[0,1]$. The recursion ends when the currently considered sub-interval $(w^1,w^2]$ satisfies one the following conditions: (i) the probability of $A \times (w^1,w^2]$  is guaranteed to be at most $\mathcal{O}(\epsilon)$ with high probability, or (ii) $w^2 - w^1 \leq 1 / K$. The \textsc{BinS} procedure returns the set of all the sub-intervals on which the recursion ended. These are guaranteed to form a partition of $[0,1]$ by construction, after adding the degenerate interval containing the sole element zero. Notice that, to determine whether condition (i) above holds or not, \textsc{BinS} calls \textsc{MCE} for every extreme of the sub-intervals $(w^1,w^2]$ (Line~\ref{line:call_mce}).
%


The recursive binary search implemented by \textsc{BinS} forms the basis for the identification of a representative family of hyperrectangles, as described in the following sections. The following lemma formally introduces the three properties guaranteed by \textsc{BinS}. Intuitively, the first property is about the accuracy of the estimates produced by the calls to \textsc{MCE}. The second property states that each $(j+1)$-dimensional hyperrectangle $A \times I$ with $I \in \mathcal{I}$ either (i) has probability of order $\mathcal{O}(\epsilon)$ or (ii) is ``thin'' along the $(j+1)$-th dimension. Finally, the third property establishes a bound on the number of sub-intervals in the partition $\mathcal{I}$ returned by \textsf{BinS}. Notably, this bound depends on the probability associated with the hypperrectangle $A$.

\begin{lemma} \label{thm:binS}
\Cref{alg:bins} guarantees that, with probability at least $1-\delta$, the returned partition $\mathcal{I}$ satisfies the following three properties.
\begin{enumerate}
    \item For every extreme point $w\in \textnormal{\textsc{Extremes}}(\mathcal{I})$ of an interval in $\mathcal{I}$, it holds:
    \[\left| \widehat{P}(A,w)-\mathbb{P}_{X \sim \cD_{j+1}}(X\in A\times [0,w])\right|\le \epsilon \, 2^{n-1}\; \sqrt{\frac{\ln(\nicefrac{4K}{\delta})}{2}} .\]
    \item For every interval $I \in \mathcal{I}$ such that $I = (w^1, w^2]$, it holds:
    \[
    \mathbb{P}_{X \sim \cD_{j+1}}({X} \in A \times I) \le \epsilon+\epsilon \, 2^{n+1} \; \sqrt{\frac{\ln(\nicefrac{4K}{\delta})}{2}} \quad \vee \quad w^2-w^1\le \frac 1 K .
    \]
    \item The cardinality of $\mathcal{I}$ is such that \( |\mathcal{I}| \le \frac 2 \epsilon  \mathbb{P}_{X \sim \cD_j}(X \in A)\log_2 K +2\).
\end{enumerate}
Moreover, \Cref{alg:bins} employs at most $\frac{1}{\epsilon^2}|\mathcal{I}|$ queries. 
\end{lemma}

\begin{proof}


As a first step, we upper bound the number of calls made to \textsc{MCE} by noticing that each call corresponds to an endpoint of an interval in \( \mathcal{I} \).
Since the stopping condition \( w^2 - w^1 \le 1/K \) ensures that \( |\mathcal{I}| \le K + 1 \le 2K \),
it follows that the total number of calls to \textsc{MCE} is bounded by \( 2K \).

Next, let us define a clean event \( \mathcal{E} \) under which every call to \textsc{MCE} made by \textsc{BinS} satisfies the property in \Cref{theo: alg MCE} for a confidence parameter \( \delta' \coloneq \delta / (2K) \). Given that $\mathcal{E}$ is simply the intersection of at most $2K$ events, by a union bound \( \mathcal{E} \) holds with probability at least $1 -\delta$.

Then, we have that, under the clean event $\mathcal{E}$ the following holds:
%
%
\[\left\lvert \widehat{P}(A,w)-\mathbb{P}_{X\sim \cD_{j+1}}(X\in A\times[0,w])\right\rvert\le \epsilon\,2^{n-1}\;\sqrt{\frac{\ln(\nicefrac{4K}{\delta})}{2}} \quad \forall w\in \textnormal{\textsc{Extremes}}(\mathcal{I}),\]
which is exactly the first property in the statement of the lemma.
%

In the following, for ease of presentation, let $\xi\coloneq\epsilon\, 2^{n-1}\sqrt{\frac{\ln(\nicefrac{4K}{\delta})}{2}}$.
%
Then, for all $i \in \mathcal{I}$ such that $I = (w^1,w^2]$, under the event $\mathcal{E}$ it holds:

\begin{align*}
    \widehat{P}(A,w^2) - \widehat{P}(A,w^1) - 2\,\mathrm{\xi} 
    &= \big(\widehat{P}(A,w^2) + \mathrm{\xi}\big)
      - \big(\widehat{P}(A,w^1) - \mathrm{\xi}\big) - 4\,\mathrm{\xi}\nonumber\\
    &\ge \mathbb{P}_{X\sim \cD_{j+1}}(X \in A \times [0,w^2])
      - \mathbb{P}_{X\sim \cD_{j+1}}(X \in A \times [0,w^1]) - 4\,\mathrm{\xi}\nonumber\\
    &= \mathbb{P}_{X\sim \cD_{j+1}}(X \in A \times (w^1,w^2])- 4\,\mathrm{\xi}\label{eq: phat difetto},
\end{align*}
and similarly:
\begin{align*}
    \widehat{P}(A,w^2) - \widehat{P}(A,w^1) - 2\,\mathrm{\xi} & = (\widehat{P}(A,w^2)-\mathrm{\xi})- (\widehat{P}(A,w^1)+\mathrm{\xi})\nonumber\\
    & \le \mathbb{P}_{X\sim \cD_{j+1}}(X \in A \times [0,w^2])
      - \mathbb{P}_{X\sim \cD_{j+1}}(X \in A \times [0,w^1]) \nonumber\\
    &= \mathbb{P}_{X\sim \cD_{j+1}}(X \in A \times (w^1, w^2])\label{eq: phat eccesso}.
\end{align*}
Hence, for all $I \in \mathcal{I}$ such that $I = (w^1,w^2]$ and $w^2-w^1>1/K$, under the clean event $\mathcal{E}$ it holds:
\begin{align*}
    \mathbb{P}_{X\sim \cD_{j+1}}(X \in A \times (w^1, w^2]) &\le \left( \widehat{P}(A,w^2) - \widehat{P}(A,w^1) - 2\xi \right) + 4\xi< \epsilon + 4\xi.
\end{align*}
This proves the second property in the statement of the lemma.



We now focus on the third property in the statement of the lemma.  
Let $n_k \in \mathbb{N}$ be the number of sub-intervals $(w^1, w^2]$ of length $w^2 - w^1 = 2^{-k}$ that are further split by the recursive procedure.  
Since an interval is split only if its length exceeds $1/K$, it follows that $k \le \log_2 K$.  
Moreover, an interval is split only if it satisfies
\[
 \widehat{P}(A,w^2) - \widehat{P}(A,w^1) - 2\,\xi 
   \ge \epsilon.
\]
Thus, under the event $\mathcal{E}$, we have:
\[
\mathbb{P}_{X\sim \cD_{j+1}}(X \in A \times (w^1, w^2]) \ge  \widehat{P}(A,w^2) - \widehat{P}(A,w^1) -2\xi \ge \epsilon.
\]
Since $\mathcal{I}$ forms a partition of $[0,1]$, it follows that
\[
n_k \, \epsilon \le \mathbb{P}_{X\sim \cD_{j}}(X \in A).
\]
Each split introduces one new interval in $\mathcal{I}$ (\emph{i.e.}, it replaces one interval with two), and thus
\[
|\mathcal{I}| 
   \le \sum_{k=0}^{\log_2 K} n_k +2
   \le (\log_2 K+1)\,\frac{\mathbb{P}_{X\sim \cD_{j}}(X \in A)}{\epsilon} +2,
\]
where the plus two is to include the interval $\{0\}$ and the interval $(0,1]$ even if $\mathbb{P}_{X\sim \cD_{j}}(X \in A)=0$.

Finally, the bound on the total number queries directly follows from the fact that \textsc{MCE} is invoked $|\mathcal{I}|$ times, and each call requires at most $1/\epsilon^2 $ samples.
\end{proof}

\subsection{Representative Families of Intervals}\label{sec: repre family}

In this section, we introduce one of the central ideas underlying our algorithm. Intuitively, this is based on the observation that, given any partition of $[0,1]$ into sub-intervals, it is possible to construct a suitable representative family of intervals---whose elements are obtained by merging consecutive intervals taken from the original partition---so as to satisfy the following two properties:
%
\begin{itemize}
    \item The representative familiy has cardinality of the same order as the original partition. Specifically, it contains at most twice as many intervals.
    \item Every interval of the form $[0,x]$ that can be expressed as a union of intervals from the original partition can also be represented as the union of at most a \emph{logarithmic} number (in the size of the original partition) of disjoint intervals from the representative family.
\end{itemize}

The observation above has an important consequence in our problem. Indeed, suppose that, for each interval in the original partition, one has access to an estimate of the probability mass associated with it. 
Then, an estimate of the probability associated with any interval of the form $[0,x]$, whenever \emph{not} directly available, 
can be computed as the sum of the estimates corresponding to a set of disjoint intervals whose union gives $[0,x]$ exactly.
However, the accuracy of the resulting estimate depends on the number of terms in such a union, as estimation errors are additive. Therefore, the capability of representing any interval $[0,x]$ as the union of only a logarithmic number of 
intervals---while using a family with size of the same order as that of the original partition---constitutes a substantial improvement over a straightforward approach based solely on the original partition.

Next, we show how to construct the desired representative family of intervals by exploiting an underlying binary encoding of the intervals in the original partition, and we prove its guarantees.
Given a partition $\mathcal{I}=\{I_1,\ldots, I_m\}$ of $[0,1]$ into $m$ sub-intervals, for each $\ell \in \{0,\ldots,\lfloor\log_2 m\rfloor\}$, we first introduce a partition $\mathcal{I}_\ell$ of $[0,1]$ defined as follows:
%
\begin{align}\label{eq:setConstructionLevel}
\mathcal{I}_\ell \coloneqq \left\{\bigcup_{k=2^\ell q+1}^{2^\ell(q+1)}I_{k} \; \Big\mid \; q=0,\ldots,\lfloor m/{2^\ell}\rfloor-1\right\} .
\end{align}
Each set $\mathcal{I}_\ell$ has cardinality $\lfloor m/2^\ell\rfloor$. Moreover, the element of $\mathcal{I}_\ell$ associated with $q$ is simply obtained by merging all the consecutive intervals $I_k$ with $k$ ranging from $2^\ell q+1$ to $2^\ell (q+1)$. Notice that the partition $\mathcal{I}_0$ actually coincides with the original partition $\mathcal{I}$.

Then, the \emph{representative family of intervals} $\mathcal{I}^\star$ is defined as the union over all these partitions:  
\begin{equation}
    \mathcal{I}^\star \coloneqq \bigcup_{\ell=0}^{\lfloor \log_2 m \rfloor} \mathcal{I}_\ell. \label{eq:setConstruction2}
\end{equation}
Clearly, the cardinality of $\mathcal{I}^\star$ is at most $2m$.
Figure~\ref{fig: tree new} shows an example of representative family, with an explanation of how it is constructed by exploiting the underlying binary encoding.

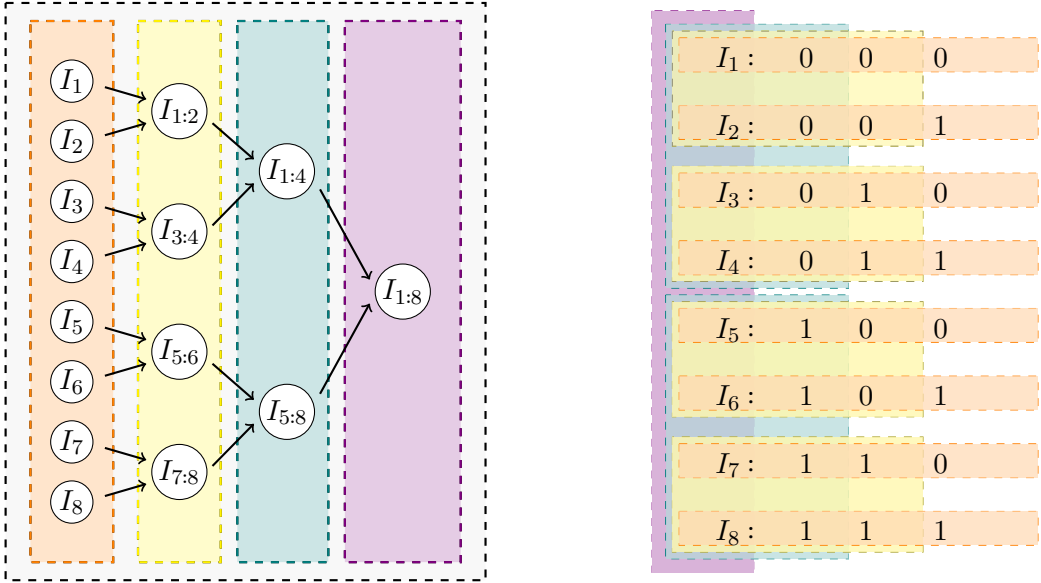
\begin{figure}[h]
    \begin{minipage}{0.55\linewidth}
    \centering
\begin{tikzpicture}[xscale=1.1,yscale=0.8,auto,rotate=-90]

\draw[dashed, thick, fill= lightgray, opacity=0.1] (-1.3,-6.8) rectangle (8.3,-1);
\draw[dashed, thick, color=black] (-1.3,-6.8) rectangle (8.3,-1);

\draw[dashed, thick, fill=white] (-1,-6.5) rectangle (8,-5.5);
\draw[dashed, thick, fill=orange, opacity=0.2] (-1,-6.5) rectangle (8,-5.5);
\draw[dashed, thick, color=orange] (-1,-6.5) rectangle (8,-5.5);

\draw[dashed, thick, fill= white] (-1,-5.2) rectangle (8,-4.2);
\draw[dashed, thick, fill= yellow, opacity=0.2] (-1,-5.2) rectangle (8,-4.2);
\draw[dashed, thick, color= yellow] (-1,-5.2) rectangle (8,-4.2);

\draw[dashed, thick, color= black,fill= white] (-1,-4) rectangle (8,-2.9);
\draw[dashed, thick, color= black,fill= teal, opacity= 0.2] (-1,-4) rectangle (8,-2.9);
\draw[dashed, thick, color= teal] (-1,-4) rectangle (8,-2.9);

\draw[dashed, thick, fill= white] (-1,-2.7) rectangle (8,-1.3);
\draw[dashed, thick, fill= violet, opacity=0.2] (-1,-2.7) rectangle (8,-1.3);
\draw[dashed, thick, color= violet] (-1,-2.7) rectangle (8,-1.3);

    \node [circle, draw=black, fill=white, inner sep=1pt] at (0, -6) {$I_{1}$};
     \node [circle, draw=black, fill=white, inner sep=1pt] at (1, -6) {$I_{2}$};
      \node [circle, draw=black, fill=white, inner sep=1pt] at (2, -6) {$I_{3}$};
       \node [circle, draw=black, fill=white, inner sep=1pt] at (3, -6) {$I_{4}$};
        \node [circle, draw=black, fill=white, inner sep=1pt] at (4, -6) {$I_{5}$};
 \node [circle, draw=black, fill=white, inner sep=1pt] at (5, -6) {$I_{6}$};
  \node [circle, draw=black, fill=white, inner sep=1pt] at (6, -6) {$I_{7}$};
   \node [circle, draw=black, fill=white, inner sep=1pt] at (7, -6) {$I_{8}$};

  \node [circle, draw=black, fill=white, inner sep=1pt] at (0.5, -4.7) {$I_{1:2}$}; 
  \node [circle, draw=black, fill=white, inner sep=1pt] at (2.5, -4.7) {$I_{3:4}$}; 
  \node [circle, draw=black, fill=white, inner sep=1pt] at (4.5, -4.7) {$I_{5:6}$}; 
  \node [circle, draw=black, fill=white, inner sep=1pt] at (6.5, -4.7) {$I_{7:8}$}; 

  \node [circle, draw=black, fill=white, inner sep=1pt] at (1.5, -3.4) {$I_{1:4}$}; 
  \node [circle, draw=black, fill=white, inner sep=1pt] at (5.5, -3.4) {$I_{5:8}$};

  \node [circle, draw=black, fill=white, inner sep=1pt] at (3.5, -2) {$I_{1:8}$};

  \draw[thick, ->] (0.1,-5.6) -- (0.3,-5.1);
\draw[thick, ->] (0.9,-5.6) -- (0.7,-5.1);
  
   \draw[thick, ->] (2.1,-5.6) -- (2.3,-5.1);
\draw[thick, ->] (2.9,-5.6) -- (2.7,-5.1);

\draw[thick, ->] (4.1,-5.6) -- (4.3,-5.1);
\draw[thick, ->] (4.9,-5.6) -- (4.7,-5.1);

\draw[thick, ->] (6.1,-5.6) -- (6.3,-5.1);
\draw[thick, ->] (6.9,-5.6) -- (6.7,-5.1);

\draw[thick, ->] (0.7,-4.3) -- (1.3,-3.8);
\draw[thick, ->] (2.4,-4.3) -- (1.7 ,-3.8);

\draw[thick, ->] (4.7,-4.3) -- (5.3,-3.8);
\draw[thick, ->] (6.4,-4.3) -- (5.7 ,-3.8);

\draw[thick, ->] (1.8,-3) -- (3.3 ,-2.4);
\draw[thick, ->] (5.2,-3) -- (3.7 ,-2.4);

\end{tikzpicture}

\end{minipage}
\begin{minipage}{0.40\textwidth}
\centering
    \begin{tikzpicture}[scale=0.9]

\draw[color=violet,dashed] (0.6,0.5) rectangle (2.1,-7.8);
\fill[color=violet!30!white] (0.6,0.5) rectangle (2.1,-7.8);

\draw[color=teal,dashed] (0.8,0.3) rectangle (3.5,-3.6);
\fill[color=teal!30!white,opacity=0.7] (0.8,0.3) rectangle (3.5,-3.6);

\draw[color=teal,dashed] (0.8,-3.7) rectangle (3.5,-7.6);
\fill[color=teal!30!white,opacity=0.7] (0.8,-3.7) rectangle (3.5,-7.6);

\draw[color=yellow,dashed] (0.9,0.2) rectangle (4.6,-1.5);
\fill[color=yellow!40!white,opacity=0.7] (0.9,0.2) rectangle (4.6,-1.5);
\draw[color=gray,dashed,  opacity=0.8] (0.9,0.2) rectangle (4.6,-1.5);

\draw[color=yellow,dashed] (0.9,-1.8) rectangle (4.6,-3.5);
\draw[color=gray, dashed, opacity=0.8] (0.9,-1.8) rectangle (4.6,-3.5);
\fill[color=yellow!40!white,opacity=0.8] (0.9,-1.8) rectangle (4.6,-3.5);

\draw[color=yellow,dashed] (0.9,-3.8) rectangle (4.6,-5.5);
\draw[color=gray,  opacity=0.8,dashed] (0.9,-3.8) rectangle (4.6,-5.5);
\fill[color=yellow!40!white,opacity=0.8] (0.9,-3.8) rectangle (4.6,-5.5);

\draw[color=yellow,dashed] (0.9,-5.8) rectangle (4.6,-7.5);
\draw[color=gray,  opacity=0.8,dashed] (0.9,-5.8) rectangle (4.6,-7.5);
\fill[color=yellow!40!white,opacity=0.8] (0.9,-5.8) rectangle (4.6,-7.5);

\draw[color=orange, dashed] (1,0.1) rectangle (6.3,-0.4);
\draw[color=orange, dashed] (1,-0.9) rectangle (6.3,-1.4);
\draw[color=orange, dashed] (1,-1.9) rectangle (6.3,-2.4);
\draw[color=orange, dashed] (1,-2.9) rectangle (6.3,-3.4);
\draw[color=orange, dashed] (1,-3.9) rectangle (6.3,-4.4);
\draw[color=orange, dashed] (1,-4.9) rectangle (6.3,-5.4);
\draw[color=orange, dashed] (1,-5.9) rectangle (6.3,-6.4);
\draw[color=orange, dashed] (1,-6.9) rectangle (6.3,-7.4); 
\fill[orange!30!white, opacity=0.7] (1,0.1) rectangle (6.3,-0.4);
\fill[orange!30!white, opacity=0.7] (1,-0.9) rectangle (6.3,-1.4);
\fill[orange!30!white, opacity=0.7] (1,-1.9) rectangle (6.3,-2.4);
\fill[orange!30!white, opacity=0.7] (1,-2.9) rectangle (6.3,-3.4);
\fill[orange!30!white, opacity=0.7] (1,-3.9) rectangle (6.3,-4.4);
\fill[orange!30!white, opacity=0.7] (1,-4.9) rectangle (6.3,-5.4);
\fill[orange!30!white, opacity=0.7] (1,-5.9) rectangle (6.3,-6.4);
\fill[orange!30!white, opacity=0.7] (1,-6.9) rectangle (6.3,-7.4); 

 \node[anchor=north west, inner sep=0.75pt] at (1.5,0) {$I_1: \hspace{0.5cm} 0\hspace{0.6cm}0\hspace{0.8cm}0$};
    \node[anchor=north west, inner sep=0.75pt] at (1.5,-1) {$I_2: \hspace{0.5cm} 0\hspace{0.6cm}0\hspace{0.8cm}1$};
    \node[anchor=north west, inner sep=0.75pt] at (1.5,-2) {$I_3: \hspace{0.5cm} 0\hspace{0.6cm}1\hspace{0.8cm}0$};
    \node[anchor=north west, inner sep=0.75pt] at (1.5,-3) {$I_4: \hspace{0.5cm} 0\hspace{0.6cm}1\hspace{0.8cm}1$};
    \node[anchor=north west, inner sep=0.75pt] at (1.5,-4) {$I_5: \hspace{0.5cm} 1\hspace{0.6cm}0\hspace{0.8cm}0$};
    \node[anchor=north west, inner sep=0.75pt] at (1.5,-5) {$I_6: \hspace{0.5cm} 1\hspace{0.6cm}0\hspace{0.8cm}1$};
    \node[anchor=north west, inner sep=0.75pt] at (1.5,-6) {$I_7: \hspace{0.5cm} 1\hspace{0.6cm}1\hspace{0.8cm}0$};
    \node[anchor=north west, inner sep=0.75pt] at (1.5,-7) {$I_8: \hspace{0.5cm} 1\hspace{0.6cm}1\hspace{0.8cm}1$};


\end{tikzpicture}
\end{minipage}

\caption{On the left, the hierarchical construction of the representative family of intervals $\mathcal{I}^\star$, built from the partition $\mathcal{I}=\{I_i\}_{i=1}^8$. 
The base $\mathcal{I}_0$ (orange) is the original partition. Higher levels $\mathcal{I}_1$ (yellow), $\mathcal{I}_2$ (blue), and $\mathcal{I}_3$ (violet) merge consecutive intervals, forming the full family $\mathcal{I}^\star$ (grey outline). On the right, the binary representation of the intervals in the family. Specifically, each \( \mathcal{I}_\ell \) groups together intervals whose binary encoding coincides in their \( n-\ell\) most significant bits.
}\label{fig: tree new}

\end{figure}

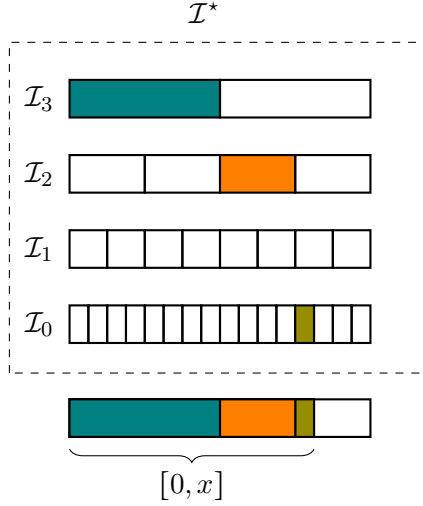
\begin{figure}[h] 
\begin{center}
    
\begin{tikzpicture}[scale=4]

\draw[thick, fill= teal] (0,7/8) rectangle (1/2,6/8);
\draw[thick] (1/2,7/8) rectangle (1,6/8); 
\node[] at (-0.1,6.5/8) {$\mathcal{I}_3$};

\draw[thick] (0,4/8) rectangle (1/4,5/8);
\draw[thick] (1/4,4/8) rectangle (2/4,5/8);
\draw[thick, fill=orange] (2/4,4/8) rectangle (3/4,5/8);
\draw[thick] (3/4,4/8) rectangle (4/4,5/8);
\node[] at (-0.1,4.5/8) {$\mathcal{I}_2$};

\draw[thick] (0,2/8) rectangle (1/8,3/8);
\draw[thick] (1/8,2/8) rectangle (2/8,3/8);
\draw[thick] (2/8,2/8) rectangle (3/8,3/8);
\draw[thick] (3/8,2/8) rectangle (4/8,3/8);
\draw[thick] (4/8,2/8) rectangle (5/8,3/8);
\draw[thick] (5/8,2/8) rectangle (6/8,3/8);
\draw[thick] (6/8,2/8) rectangle (7/8,3/8);
\draw[thick] (7/8,2/8) rectangle (8/8,3/8);
\node[] at (-0.1,2.5/8) {$\mathcal{I}_1$};

\draw[thick] (0,0) rectangle (1/16,1/8);
\draw[thick] (1/16,0) rectangle (2/16,1/8);
\draw[thick] (2/16,0) rectangle (3/16,1/8);
\draw[thick] (3/16,0) rectangle (4/16,1/8);
\draw[thick] (4/16,0) rectangle (5/16,1/8);
\draw[thick] (5/16,0) rectangle (6/16,1/8);
\draw[thick] (6/16,0) rectangle (7/16,1/8);
\draw[thick] (7/16,0) rectangle (8/16,1/8);
\draw[thick] (8/16,0) rectangle (9/16,1/8);
\draw[thick] (9/16,0) rectangle (10/16,1/8);
\draw[thick] (10/16,0) rectangle (11/16,1/8);
\draw[thick] (11/16,0) rectangle (12/16,1/8);
\draw[thick, fill= olive] (12/16,0) rectangle (13/16,1/8);
\draw[thick] (13/16,0) rectangle (14/16,1/8);
\draw[thick] (14/16,0) rectangle (15/16,1/8);
\draw[thick] (15/16,0) rectangle (16/16,1/8);
\node[] at (-0.1,0.5/8) {$\mathcal{I}_0$};

\draw[dashed] (-0.2,-0.1) rectangle (1.2,1);

\node[] at (0.45,1.1) {$\mathcal{I}^\star$};

\draw[thick] (0,-2.5/8) rectangle (13/16,-1.5/8);

\draw[thick] (0,-2.5/8) rectangle (1,-1.5/8);

\draw[thick, fill= olive] (12/16,-2.5/8) rectangle (13/16,-1.5/8);
\draw[thick, fill= orange] (2/4,-2.5/8) rectangle (12/16,-1.5/8);
\draw[thick, fill= teal] (0,-2.5/8) rectangle (1/2,-1.5/8);

\draw [decorate,decoration={brace,amplitude=5pt,mirror,raise=4ex}]
  (0,-1.5/8) -- (13/16,-1.5/8) node[midway,yshift=-3em]{$[0,x]$};

\end{tikzpicture}
\caption{Graphical representation of the representative family $\mathcal{I}^\star$ and its use in expressing a generic interval $[0,x]$. In this example, $[0,x] = \sum_{k=1}^{13} I_k$. Since $13$ has binary representation $1101$, the union of the first $13$ intervals can be expressed as the union of one interval from each $\mathcal{I}_\ell$ corresponding to a $1$ in the $(\ell{+}1)$-th least significant bit. In this case, the ones occur in positions $1$, $3$, and $4$, meaning that the representation within $\mathcal{I}^\star$ includes one element from each of $\mathcal{I}_0$, $\mathcal{I}_2$, and $\mathcal{I}_3$.}\label{fig: bin rep square}
\end{center}
\end{figure}

\begin{lemma}\label{thm:logSize}
    Let $\mathcal{I} = \{I_1, \ldots, I_m\}$ be a collection of $m \in \mathbb{N}$ disjoint intervals whose union is $[0,1]$, and let $\mathcal{I}^*$ be constructed according to Equations~\eqref{eq:setConstruction2}.
    Then, the following holds:
    \begin{itemize}
        \item $\mathcal{I}^\star$ has cardinality at most $2m$.
        \item For any $x \in [0,1]$, by letting $I_k\in \mathcal{I}$ be the interval such that $x\in I_k$, there exists a set $\overline{\mathcal{I}} \subseteq \mathcal{I}^\star$ of disjoint intervals of cardinality $\mathcal{O}(\log m)$ such that 
                \[ \bigcup_{I \in \overline{\mathcal{I}}} I \subseteq [0,x] \subseteq \left(\bigcup_{I \in \overline{\mathcal{I}}} I\right) \cup I_k . \]
    \end{itemize}
\end{lemma}

\begin{proof}
The proof is organized in two parts, following the structure of the statement.  
First, we establish an upper bound on the cardinality of the representing family of intervals \(\mathcal{I}^\star\).  
Then, we construct, for an arbitrary \(x \in [0,1]\), the corresponding set \(\overline{\mathcal{I}} \subseteq \mathcal{I}^\star\) associated with the interval \([0, x]\), and show its approximation guarantees and cardinality.

\paragraph{Cardinality of $\mathcal{I}^\star$}
By construction, it holds:
\[
|\mathcal{I}^\star|
= \sum_{\ell=0}^{\lfloor \log_2 m \rfloor} |\mathcal{I}_\ell|
= \sum_{\ell=0}^{\lfloor \log_2 m \rfloor} \lfloor m / 2^\ell \rfloor
\le 2m.
\]
Hence, we have $|\mathcal{I}^\star| \leq 2m$.

\paragraph{Construction of $\overline{\mathcal{I}}$}
Fix any $x \in [0,1]$, and let $k$ be the unique index such that $x \in I_k$.
We now introduce a sequence of indices $\{k_j\}$, which are recursively defined as follows:
\begin{itemize}
    \item $k_0 = k - 1$,
    \item while $k_j > 1$, let $p_j$ be the largest integer $p$ such that $2^p$ divides $k_j$, and define $k_{j+1} = k_j - 2^{p_j}$.
\end{itemize}

Let $J$ be the length of the sequence $\{k_j\}$ constructed above. 
The associated sequence $\{p_j\}$, of the same length $J$, is strictly increasing, 
since after each subtraction the highest power of two dividing the next index $k_{j+1}$ is strictly larger than that dividing $k_j$ (equivalently, $k_{j+1}$ has more trailing zeros in its binary representation than $k_j$).

To see this more concretely, consider the binary representation of $k_j$. 
Starting from the least significant bit (position $0$), 
let $p$ be the index of the first bit equal to $1$. 
This means that $2^p$ is the largest power of two dividing $k_j$. 
Subtracting $2^p$ from $k_j$ replaces that $1$ with a $0$ in the binary representation, 
thereby increasing the number of trailing zeros and consequently the power of two dividing the result. 
Hence, $p_{j+1} > p_j$ for all $j$, and $J \le  \log_2 m $, as $p_j\le \log_2(m)$ for all $j$.


Finally, define
\[
\overline{\mathcal{I}}
= \bigcup_{j=1}^{J} \left( \bigcup_{z = k_{j-1}+1}^{k_j} I_z \right)
= \bigcup_{z=1}^{k-1} I_z.
\]
By definition,
\[
\bigcup_{I \in \overline{\mathcal{I}}} I 
= \bigcup_{z=1}^{k-1} I_z 
\subseteq [0,x] 
\subseteq \bigcup_{z=1}^{k} I_z
= \left(\bigcup_{I \in \overline{\mathcal{I}}} I\right) \cup I_k.
\]

Moreover, each block of the form \(\bigcup_{z=k_{j-1}+1}^{k_j} I_z\) belongs to \(\mathcal{I}_{p_j}\).
Indeed, by construction \(k_j\) is divisible by \(2^{p_j}\). Thus, there exists an integer \(q\) such that \(k_j = q \, 2^{p_j}\). From the definition of \(k_{j-1}\), we have 
\(k_{j-1} = q \, 2^{p_j} - 2^{p_j} = (q-1) 2^{p_j}\). Setting \(q' = q - 1\), it follows that
\[
\bigcup_{z=k_{j-1}+1}^{k_j} I_z= \bigcup_{z=q'(2^{p_j})+1}^{(q'+1)2^{p_j}} I_z \in \mathcal{I}_{p_j}.
\]
Consequently, \(\overline{\mathcal{I}}\) is the union of at most \(\log_2 m\) elements of \(\mathcal{I}^\star\), which proves the claim.
\end{proof}

\subsection{Representative Hyperrectangles Identification} \label{sec: fine fase 1}

We conclude by presenting the main procedure executed by our algorithm, which is the \textsc{Representative Hyperrectangles identification (\textsc{RHI})} procedure. Intuitively, it employs the \textsc{BinS} procedure and a representative family of intervals constructed as in \Cref{eq:setConstruction2} to iteratively build hyperrectangles one dimension at a time. Specifically, given a hyperrectangle of dimension $j$, it builds $(j+1)$-dimensional hyperrectangles by partitioning the $(j+1)$-th dimension by means of \textsc{BinS} to get $\mathcal{I}_0(A)$ and constructing its resulting representative family of intervals $\mathcal{I}^\star(A)$.
%

The \textsc{RHI} procedure (presented in Algorithm~\ref{alg: map sync}) proceeds from dimension $j = 0$ to dimension $j = n-1$. For each dimension $j$, the procedure considers each hyperrectangle $A \in \mathcal{R}^j$ of dimension $j$, where $\mathcal{R}^j$ is the set of all the hyperrectangles of dimension $j$ built in the previous step.
%
%
Then, for each of them, it applies the procedure \textsc{BinS} to partition the $(j+1)$-th dimension, so as to build suitable $(j+1)$-dimensional hyperrectangles. This gives a partition into sub-intervals of $[0,1]$, which is stored in $\mathcal{I}_0(A)$. Then, the procedure builds the partitions $\mathcal{I}_\ell(A)$ according to \Cref{eq:setConstructionLevel}, by considering $\mathcal{I}_0(A)$ as the original set of intervals $\mathcal{I}$, and it also constructs the representative family of intervals $\mathcal{I}^\star(A)$ as their union over $\ell$---that is, $\mathcal{I}^\star(A) \coloneq \bigcup_{\ell=0}^{\log_2 |\mathcal{I}_0(A)|}\mathcal{I}_\ell(A)$.




A visual and intuitive example of how \textsc{RHI} operates in $n\ge2$ dimensions is shown in the right-hand side of Figure~\ref{fig: rep new}.  
Consider the projection onto the first two coordinates, \(X_1\) and \(X_2\).  
The top-left panel depicts the elements of \(\mathcal{R}^1\), corresponding to sets of the form \( X_1 \in I \) for \(I \in \mathcal{I}_0(\mathbbm{1})\).  
The top-right panel illustrates the elements of \(\mathcal{R}^2\), generated by hyperrectangles of the form \(I_1 \times I_2\), where \(I_1 \in \mathcal{I}_1(\mathbbm{1})\) and \(I_2 \in \mathcal{I}_0(I_1)\).  
Similarly, the bottom-left panel displays the elements derived from \(A \in \mathcal{R}^1\), defined by intervals in \(\mathcal{I}_2(\mathbbm{1})\) and \(\mathcal{I}_3(\mathbbm{1})\), along with the corresponding refinements on the second coordinate given by \(\mathcal{I}_0(A)\).  
The procedure for the other dimensions goes on analogously: at each step, the algorithm refines a single coordinate, thereby reducing the construction to a sequence of one-dimensional subdivisions.

\begin{algorithm}[!htp]
\caption{\textsc{Representative Hyperrectangles Identification (RHI)}}
\label{alg: map sync}
\begin{algorithmic}[1]
    \Require Accuracy $\epsilon >0 $; confidence $\delta \in (0,1)$; uniform grid $\mathcal{G}_K \subseteq [0,1]^n$ of resolution $1/K$ ($K \in \mathbb{N}_+$)
    \Ensure The properties stated in \Cref{theo: MC inef} are satisfied
    \Statex \hrulefill
    \Function{RHI}{$\epsilon, \delta, \cG_K$}
    \State $\mathcal{R}^0 \gets \{ \mathbbm{1} \} $; $\mathcal{R}^j \gets \varnothing$ for all $j \in [n]$ 
    \For{$j = 0, \ldots, n-1$}
        \ForAll{$A \in \mathcal{R}^{j}$}
         \State $\mathcal{I}_{0}(A) \gets \textsc{BinS}(A, \epsilon, \delta/(4K)^n)$
         \State Build partitions $\mathcal{I}_\ell(A)$ for $\ell = 1, \ldots, \lfloor \log_2 |\mathcal{I}_0(A)| \rfloor$ according to \Cref{eq:setConstructionLevel}
         \State Build $\mathcal{I}^\star(A)$ as the union of the sets $\mathcal{I}_\ell(A)$ according to \Cref{eq:setConstruction2}
         \ForAll{$I \in \mathcal{I}^\star(A)$}
            \State $\mathcal{R}^{j+1} \gets \mathcal{R}^{j+1} \cup \{ A \times I \}$
         \EndFor
         %
        \EndFor
        \EndFor
        \EndFunction
\end{algorithmic}
\end{algorithm}

One of the crucial properties that allows us to get a small representative family of hyperrectangles is that the cumulative probability over the space $[0,1]^n$ sums to one. Thus, for each marginalized probability distribution $\cD_j$, there are at most $1/\epsilon$ hyperrectangles that can have associated probability greater than or equal to $\epsilon$. Indeed, as shown in \Cref{thm:binS}, the number of intervals returned by \textsc{BinS} depends on the probability of the hyperrectangle.

\begin{lemma} \label{theo: MC inef} 
    \Cref{alg: map sync} guarantees that, with probability at least $1-\delta$, the following four properties are satsfied.
    \begin{itemize}
    \item For each dimension $j\in [n-1]$, hyperrectangle $A \in \mathcal{R}^{j}$, and interval $I \coloneq (w^1,w^2]\in \mathcal{I}_0(A)$ such that $w^2-w^1> 1 / K$, it holds:
\[\mathbb{P}_{X\sim \cD_{j+1}}(X\in A \times I)\le \epsilon +  \epsilon\;\left(2^{n+1}\,\sqrt{n\ln(\nicefrac{4K}{\delta})}\right).\]
\item  For each dimension $j \in [n-1]$, hyperrectangle $A \in \mathcal{R}^{j}$, and extreme point $w\in \textnormal{\textsc{Extremes}}(\mathcal{I}_0(A))$, it holds:
\[\left\lvert\mathbb{P}_{X\sim \cD_{j+1}}(X\in A\times [0,w])-\widehat{P}(A,w)\right\rvert\le 2^{n-1}\epsilon\sqrt{n\ln(\nicefrac{4K}{\delta})}.\]
\item The cardinality of the family of hyperrectangles is $\sum_{j \in [n]} |\mathcal{R}^j| \le \frac 1 \epsilon {2^{n-1}(4\log_2 K)^{n+2}}$.
\item \Cref{alg: map sync} employs at most $\frac 1 {\epsilon^3} {2^{n-1}(4\log_2 K)^{n+2}}$ queries.
\end{itemize}
\end{lemma}

\begin{proof}
The first part of the statement follows by applying \Cref{thm:binS}
to every call of Algorithm~\ref{alg:bins} and then taking a union bound over
all calls.

By construction, each set $\mathcal{I}_0(A)$ is produced by a
call to Algorithm~\ref{alg:bins} with inputs $(A,\epsilon,\delta/(4K)^n)$.
Hence, by \Cref{thm:binS}, for a fixed such call we have that with
probability at least $1-\delta/(4K)^n$ the following hold for every
$(a,b]\in \mathcal{I}_0(A)$ with $b-a>1/K$ and for all $w\in \textsc{Extremes}(\mathcal{I}(A))$ :
\begin{align}
\mathbb{P}_{X\sim \mathcal{D}^{j+1}}\!\big( X\in A \times (a,b] \big)
&\le \epsilon + \epsilon\; 2^{n+1}\sqrt{\frac{\ln(4K(4K)^n/\delta)}{2}}\nonumber\\
& \le \epsilon + \epsilon\; 2^{n+1}\sqrt{\frac{(n+1)\ln(4K/\delta)}{2}}\nonumber\\
& \le \epsilon + \epsilon\; 2^{n+1}\sqrt{n\ln(4K/\delta)}\label{eq: ev 1},
\end{align}
and
\begin{equation}|\mathbb{P}_{X\sim \mathcal{D}^{j+1}}\!\big( X\in A \times (0,w] \big)-\widehat{P}(A,w)|\le \epsilon\, 2^{n-1}\sqrt{\frac{\ln(\frac{4K(4K)^n}{\delta})}{2}}\le\epsilon\; 2^{n-1}\sqrt{n\ln(4K/\delta)}.\label{eq: ev 2}\end{equation}
Therefore,  to prove the first part of statement it is sufficient to define $\mathcal{E}$ as the event in which \Cref{eq: ev 1} and \Cref{eq: ev 2} hold for all $A\in \mathcal{R}^j$ with $j\in \{0,\ldots,n-1\}$, and show how $\mathcal{E}$ is the intersection of at most $(4K)^n$ events, each with probability at least $1-\delta/(4K)^n$. 

Since the number of events in the intersection correspond to the number of calls to Algorithm~\ref{alg:bins} performed by Algorithm~\ref{alg: map sync}, we have to show that the number of calls to Algorithm~\ref{alg:bins} is deterministically at most $(4K)^n$.  

Hence, we will now bound the number of calls to Algorithm~\ref{alg:bins} in two complementary ways:  
first, through a deterministic (worst-case) bound, which will be used to control the probability of $\mathcal{E}$;  
and second, through a tighter, data-dependent bound that holds under the event $\mathcal{E}$.

\paragraph{Deterministic (Worst-Case) Bound}
Let \(\mathcal{R}^{j}\) denote the collection of $j$-dimensional hyperrectangles built by the algorithm of dimension \(j\),
and let \(|\mathcal{R}^j|\) be its cardinality. By definition we have \(|\mathcal{R}^0|=1\).
Each element \(A\in \mathcal{R}^{j}\) generates at most \(2|\mathcal{I}_0(A)|\)
elements in \(\mathcal{R}^{j+1}\) (one for each interval in $ \mathcal{I}^\star(A)$), indeed
\[
|\mathcal{R}^{j+1}| =\sum_{A\in \mathcal{R}^{j}}|\mathcal{I}^\star(A)|= \sum_{A\in \mathcal{R}^{j}}\sum_{\ell =0}^{\log_2(|\mathcal{I}_0(A)|)}|\mathcal{I}_\ell(A)|\le \sum_{A\in \mathcal{R}^{j}}\sum_{\ell =0}^{\log_2(|\mathcal{I}_0(A)|)} 2^{-\ell}|\mathcal{I}_0(A)| \le \sum_{A\in \mathcal{R}^{j}} 2\,|\mathcal{I}_0(A)|.
\]
Since \(|\mathcal{I}_0(A)|\le 2K\) for every \(A\), we obtain the recursive
bound \(|\mathcal{R}^{j+1}|\le 4K\,|\mathcal{R}^{j}|\). Unrolling the recursion gives
\[
|\mathcal{R}^j| \le (4K)^{j-1} |\mathcal{R}^1| = (4K)^j ,
\]
since $|\mathcal{R}^1|\le 2|\mathcal{I}_0(\mathbbm{1})|\le 4K$.

Therefore, the total number of calls of Algorithm~\ref{alg:bins} is at most
\[
\sum_{j=0}^{n-1} |\mathcal{R}^j| \le \sum_{j=0}^{n-1} (4K)^j \le (4K)^n,
\]
using the geometric series and hence, in the worst case, Algorithm~\ref{alg:bins} is invoked at most
\((4K)^n\) times. Combining this with the per-call failure probability
$\delta/(4K)^n$ and applying a union bound yields the desired overall failure
probability at most
\((4K)^n \cdot \frac{\delta}{(4K)^n} = \delta\), so $\mathcal{E}$ holds with probability at least $1-\delta$.

\paragraph{High-Probability Bound}
Under $\mathcal{E}$, thanks to \Cref{thm:binS} we have the following upper bound for
every \(A\in \mathcal{R}^{j}\):
\[
|\mathcal{I}_0(A)| \le 2\log_2(K)\,\frac{\mathbb{P}(X\in A)}{\epsilon}+2.
\]
We now use this to control the total number of intervals produced . Note that
\[
\mathcal{R}^{j+1} = \{\, A\times I : A\in \mathcal{R}^{j},\ I\in \mathcal{I}_\ell(A),\ \ell=0,\dots,\log_2(|\mathcal{I}_0(A)|)\,\},
\]
and that for fixed \(A\) and $\ell$ the sets \(\mathcal{I}_\ell(A)\) form a partition
of \([0,1]\). Therefore
\begin{align*}
\sum_{B\in \mathcal{R}^{j+1}}\mathbb{P}_{X\in \cD^{j+1}}(X\in B)
&= \sum_{A\in \mathcal{R}^{j}}\sum_{\ell=0}^{\log_2(|\mathcal{I}_0(A)|)}
    \sum_{I\in\mathcal{I}_\ell(A)} \mathbb{P}_{X\sim\cD^{j+1}}(X\in A\times I) \\
&= \sum_{A\in \mathcal{R}^{j}}\sum_{\ell=0}^{\log_2(|\mathcal{I}_0(A)|)}
    \mathbb{P}_{X\in \cD^{j+1}}(X\in A\times[0,1]) \\
&= \sum_{A\in \mathcal{R}^{j}} \big(1+\log_2(|\mathcal{I}_0(A)|)\big)\,\mathbb{P}_{X\sim \cD^j}(X\in A)\\
& \le \sum_{A\in \mathcal{R}^{j}} \big(4\log_2(K)\big)\,\mathbb{P}_{X\sim \cD^j}(X\in A),
\end{align*}
where in the last inequality we used the deterministic bound $|\mathcal{I}_0(A)|\le 2K$ and the trivial inequality $1+\log_2(2u)\le 4\log_2(u)$ for $u\ge 2$.
 Therefore,
\[
\sum_{A\in \mathcal{R}^{j+1}}\mathbb{P}_{X\sim \cD^{j+1}}(X\in A)
\le 4\log_2(K)\sum_{A\in \mathcal{R}^{j}}\mathbb{P}_{X\sim \cD^j}(X\in A),
\]
 and recursing over \(j\) yields
\[
\sum_{A\in \mathcal{R}^{j}}\mathbb{P}_{X\sim \cD^j}(X\in A) \le \left(4\log_2(K)\right)^j,
\]
since \(\sum_{A\in \mathcal{R}^1}\mathbb{P}(X\in A)=\sum_{\ell=0}^{\log_2(|\mathcal{I}_0(A)|)}\sum_{I\in \mathcal{I}_\ell(\mathbbm{1})}\mathbb{P}_{X\sim \cD^1}(X\in I)\le 4\log_2(K)\).

Therefore, using the bound 
\[|\mathcal{I}_0(A)|\le 2\log_2(K)\,\mathbb{P}(X\in A)/\epsilon+2,\] we obtain that, with probability at least $1-\delta$,
we have that
\begin{align*}
    |\mathcal{R}^{j+1}|&=\sum_{A\in \mathcal{R}^{j}}\sum_{\ell=1}^{\log_2(|\mathcal{I}_0(A)|)} |\mathcal{I}_\ell(A)|\\
    & \le \sum_{A\in \mathcal{R}^{j}}2|\mathcal{I}_0(A)|\le \frac{4\log_2(K)}{\epsilon}(4\log_2(K))^j+2|\mathcal{R}^j|,
\end{align*}
and recursing over $j$, considering $|\mathcal{R}^0|=1$ we have that 
\[|\mathcal{R}^{j+1}|\le \frac{4\log_2(K)}{\epsilon}\sum_{k=0}^{j}2^{j-k}(4\log_2(K))^k\le \frac{2^j(4\log_2(K))^{j+2}}{\epsilon},\]
where the last inequality holds thank to the geometric series property $\sum_{k=0}^{n-1}r^k=(r^n-1)/(r-1)\le r^n$ if $r>2$.
Therefore, under $\mathcal{E}$
\begin{equation*}
|\mathcal{R}|=\sum_{j=0}^{n}|\mathcal{R}^j|=\sum_{j=0}^{n}\frac{2^{j-1}(4\log_2(K))^{j+1}}{\epsilon}\le \frac{2^{n-1}(4\log_2(K))^{n+2}}{\epsilon},
\end{equation*}
and it requires at most the following number of samples:
    
    \begin{align*}
                 \sum_{j=0}^{n-1}&\sum_{A\in \mathcal{R}^{j}}\sum_{\ell=0}^{\log_2(|\mathcal{I}_0(A)|)} |\mathcal{I}_\ell(A)|\frac{1}{\epsilon^2} \le \frac{2^{n-1}(4\log_2(K))^{n+2}}{\epsilon^3}.
    \end{align*}

This completes the proof.
\end{proof}

\section{From Representative Hyperrectangles to Estimates} \label{sec:estimate}

In this section, we introduce the algorithm constituting the second step of the proof of our main result.
We design an algorithm that takes as inputs all the elements computed by the \textsc{RHI} procedure, namely, a representative family of hyperrectangles defined by the sets $\mathcal{R}^j$ for $j \in [n]$, the representative families of intervals $\mathcal{I}^\star(A)$ for $A \in \mathcal{R}^j, j \in [n-1]$, and the estimates $\widehat{P}(A,w)$ for $A \in \mathcal{R}^j, j \in [n-1], w \in [0,1]$.
The algorithm employs all such elements to compute an accurate estimate of the CDF at any given point \( x \in \mathcal{G}_K \) on the grid, and, thus, it is the algorithm used to output the values of the function $\widetilde{P}$ introduced in the statement of Theorem~\ref{theo: main}. 

The estimation procedure proceeds recursively, refining the approximation of the target hyperrectangle by partitioning it along one additional dimension at each step, using the representative hyperrectangles $ \mathcal{R}^{j+1}$.  
At each iteration, this process effectively reduces a \( j \)-dimensional estimation problem to a logarithmic number of \((j-1)\)-dimensional subproblems, enabling a controlled propagation of error across dimensions.

Crucially, we leverage \Cref{thm:logSize}, which guarantees that these partitions have only logarithmic size.  
This property is essential: since the total estimation error scales linearly with the number of elements in the partition, maintaining a logarithmic partition size ensures that the cumulative error remains within the desired \( \widetilde{\mathcal{O}}(\epsilon) \) bound. 
In particular, we can partition the space \( A \times [0, \tilde x_{j+1}] \) into at most \( \log_2 K \) hyperrectangles in \( \mathcal{R}^{j+1} \), where \( \tilde x_{j+1} \) denotes the projection of \( x_{j+1} \) onto the set of extremes of the intervals in \( \mathcal{I}_0(A) \).  
Thanks to the structure enforced by the \textsc{RHI} procedure, the resulting approximation error---corresponding to the probability of \( A \times (\tilde x_{j+1}, x_{j+1}] \)---is guaranteed to be of order \( \mathcal{O}(\epsilon) \).

\begin{algorithm}[!t]
\caption{\textsc{Compute Grid Estimate} (\textsc{CGE})}
\label{alg: P tilde}
\begin{algorithmic}[1]
\Require An evaluation point $x\in \mathcal{G}_K$; a representative family of hyperrectangles $\{\mathcal{R}^j\}_{j \in [n]}$, representative families of intervals $\mathcal{I}^\star(A)$ (for $A \in \mathcal{R}^j, j \in [n-1]$), and estimates $\widehat{P}(A,w)$ (for $A \in \mathcal{R}^j, j \in [n-1], w \in [0,1]$), all computed by means of \textsc{RHI}($\delta,\epsilon, \cG_K$)
%
\Ensure The properties stated in \Cref{thm:final} are satisfied
\Statex \hrulefill
\Function{\textsc{CGE}}{$x$}
\State $\widetilde{P}(x)\gets \textsc{CGE-Rec}(\mathbbm{1},x)$
\State \Return $\widetilde{P}(x)$
\EndFunction
\Statex \hrulefill
\Function{CGE-Rec}{$A, x$}\Comment{$A \in \mathcal{R}^j$ and $x \in \cG_K$}
\State $\tilde x_{j+1} \gets $ Solution to \Cref{eq:projection} 
\If{$j+1=n$}
\State \Return $\widehat{P}(A,\tilde x_{j+1}) $ 
\Else 
\State $\mathcal{H}(A,\tilde x_{j+1})\gets \textsc{CLP}(A,\tilde x_{j+1})$
\State \Return $\sum_{B\in \mathcal{H}(A,\tilde x_{j+1})} \textsc{CGE-Rec}(B,x) $
\EndIf
\EndFunction
\end{algorithmic}
\end{algorithm}

\begin{algorithm}[!t]
\caption{\textsc{Compute Logaritmic Partition} (\textsc{CLP})} 
\label{alg: H}
\begin{algorithmic}[1]
\Require A hyperrectangle $A\in \mathcal{H}^{j}$ of dimension $j \in [n-1]$; an extreme point $x \in \textsc{Extremes}(\mathcal{I}_0(A))$ 
\Ensure $\mathcal{H}(A,x)$ is a partition of $A\times [0,x]$ made by at most $\log_2(K)$ hyperrectangles
\State $\mathcal{H}(A,x) \gets \emptyset$
\State Apply the procedure in \Cref{thm:logSize} for $\mathcal{I} = \mathcal{I}_0(A)$ in order to build a set $\overline{\mathcal{I}}$ of cardinality at most $\log_2 |\mathcal{I}_0(A)|$ of non-overlapping intervals $I \in \mathcal{I}^\star(A)$ such that $\bigcup_{I \in \overline{\mathcal{I}}} I=[0,x] $
\For {$I \in \mathcal{S}$}
\State $\mathcal{H}(A,x)\gets \mathcal{H}(A,x) \cup \{A \times I \}  $
\EndFor
\State \Return $\mathcal{H}(A,x)$
\end{algorithmic}
\end{algorithm}

More in details, the (\textsc{CGE}) procedure simply returns the value of \textsc{CGE-Rec}($\mathbbm{1}, x$). The supporting procedure \textsc{CGE-Rec} takes as inputs a hyperrectangle $A \in \mathcal{R}^j$ of dimension $j \in [n-1]$ and $x$, and it recursively computes an estimate of the probability of an $n$-dimensional hyperrectangle of the form:
%
\[
A \times \left( \bigtimes_{k=j+1}^n [0, x_k] \right).
\]
Clearly, \textsc{CGE-Rec}($\mathbbm{1}, x$) given an estimate of the desired probability $\mathbb{P}_{X \sim \cD} (X \leq x)$.
%
%
As a first step, \textsc{CGE-Rec} projects \( x_{j+1} \) onto the set of extremes of the intervals contained in the set \( \mathcal{I}_0(A) \), that is, it computes a projected point $\tilde x_{j+1}$ such that:
\begin{align} \label{eq:projection}
\tilde x_{j+1} \coloneq \max\{ z \in \textsc{Extremes}(\mathcal{I}_0(A)) \mid z \le x_{j+1} \}.
\end{align}
Then, two cases are possible:
\begin{itemize}
\item If $ j + 1 = n $, then an estimate of the probability 
\(\mathbb{P}(A \times [0, \tilde x_{j+1}])\) 
is already available as 
\(\widehat{P}(A, \tilde x_{j+1})\), 
since \( A \in \mathcal{R}^{n-1} \) and \( \tilde x_{j+1} \in \mathcal{I}_0(A) \).

\item If $ j + 1 < n $, the algorithm identifies 
\(\mathcal{O}(\log_2 K)\) 
subproblems of lower dimensionality. 
To do so, it invokes the \textsc{Compute Logarithmic Partition} (\textsc{CLP}) procedure, which returns a set 
\(\mathcal{H}(A, x_{j+1})\) 
of disjoint \((j+1)\)-dimensional hyperrectangles in 
\(\mathcal{R}^{j+1}\), 
of cardinality at most \(\log_2 K\), 
whose union exactly equals 
\(A \times [0, \tilde x_{j+1}]\). This leverages the notion of representative family introduced in Section~\ref{sec: repre family}. 
Hence, the original \((n-j)\)-dimensional estimation problem is reduced to at most \(\log_2 K\) subproblems of dimension \((n-j-1)\).
Formally,
\[
\mathbb{P}\!\left(
X \in 
A \times  
\Big( \bigtimes_{k=j+1}^n [0, x_k] \Big)
\right) \simeq
\sum_{B \in \mathcal{H}(A, \tilde x_{j+1})} 
\mathbb{P}\!\left(
X \in 
B \times 
\Big( \bigtimes_{k=j+2}^n [0, x_k] \Big)
\right).
\]
The approximation error generated at this step, whenever $x_{j+1}\neq \tilde x_{j+1}$, corresponds to the probability mass
\[
\mathbb{P}\!\left(
X \in A \times (\tilde x_{j+1}, x_{j+1}] \times 
\Big( \bigtimes_{k=j+2}^n [0, x_k] \Big),
\right)
\]
and is bounded by 
\(\widetilde{\mathcal{O}}(\epsilon)\).
Indeed, the \textsc{RHI} procedure ensures that either \( \tilde x_{j+1} = x_{j+1} \), 
or the set 
\( A \times (\tilde x_{j+1}, x_{j+1}] \) 
is contained in some 
\( A \times I \) 
with \( I \in \mathcal{I}_0(A) \) such that 
\(\mathbb{P}(A \times I) \le \widetilde{\mathcal{O}}(\epsilon)\) by design.
\end{itemize}
Finally, this procedure is applied recursively: the overall CDF 
\(\mathbb{P}(X \le x) = \mathbb{P}(X \in \bigtimes_{k=1}^n [0, x_k])\)
is approximated by a sequence of \textsc{RHI} calls, starting from the unique hyperrectangle in \(\mathcal{R}^0\).

\begin{lemma} \label{thm:final}
    Assume that the properties in \Cref{theo: MC inef} hold.
    Algorithm~\ref{alg: P tilde}  guarantees
      \[\left|\mathbb{P}(X\le x)-\widetilde{P}(x)\right|\le \epsilon\left(2^{n+3}\sqrt{{n\ln\left(\frac{4K}{\delta}\right)}}\right)(\log_2(K))^n \quad \forall x\in \mathcal{G}_K .
      \]
\end{lemma}

\begin{proof}
    Let $x=\{x_i\}_{i=1}^n$ be an element of the grid $\mathcal{G}_K$. Estimating the CDF in $x$ is equivalent to estimating the probability associated to the set:
\[\{y\in [0,1]^n: y_i\le x_i \; \forall i\in [n]\}.\]
To do so, we propose a decomposition of the space along the first direction, using Algorithm~\ref{alg: H}. Indeed,
\begin{align*}
\mathbb{P}(X\le x)& = \sum_{B\in \mathcal{H}(\mathbbm{1},\tilde x_1)}\mathbb{P}(X \in B \times ( \bigtimes_{j=2}^n [0,x_j])) + \mathbb{P}(X\in (\tilde x_1,x_1]\times ( \bigtimes_{j=2}^n [0,x_j]))
\end{align*}
And similarly, given a $j\in [n-1]$, and an $j$-dimensional hyperrectangle $A$ in $\mathcal{R}^j$
the following space decomposition holds:
\begin{align}
    \mathbb{P}(X\in A\times ( \bigtimes_{i=j+1}^n [0,x_i]))&= \sum_{B\in \mathcal{H}(A,\tilde x_{j+1})}\mathbb{P}(X\in B\times ( \bigtimes_{i=j+2}^n [0,x_i])) \nonumber\\
    &\hspace{0.5cm}+ \mathbb{P}(X \in A\times (\tilde x_{j+1}, x_{j+1}]\times ( \bigtimes_{i=j+2}^n [0,x_i])).\label{eq: aprox j}
\end{align}

\paragraph{Bound the Probability $ \mathbb{P}(X \in A\times (\tilde x_{j+1}, x_{j+1}]\times ( \bigtimes_{i=j+2}^n [0,x_i]))$.}
Here, the idea is to use how the set $\mathcal{R}^j$ is built to show that the approximation error on the $(j+1)$-th dimension is small, with $j\in \{0,\ldots,n-1\}$. To do so, we first observe that, by the monotonicity of the probability 
\begin{align*}
     \mathbb{P}(X \in A\times (\tilde x_{j+1},x_{j+1}]\times ( \bigtimes_{i=j+2}^n [0,x_i]))\le  \mathbb{P}_{X\in \cD^{j+1}}(X \in A\times (\tilde x_{j+1},x_{j+1}]).
\end{align*}
Therefore we can exploit the properties guaranteed by \Cref{theo: MC inef} to show  that one of the two hold:
\begin{itemize}
    \item $x_{j+1}= \tilde x_{j+1}$ and there is no approximation error
    \item it exist $(a,b]\in \mathcal{I}_0^{j+1}(A)$ such that $(\tilde x_{j+1}, x_{j+1}]\subseteq (a,b]$ and $b-a>1/K$, which implies
\[\mathbb{P}_{X\in \cD^{j+1}}(X \in A\times (x_{j+1},\tilde x_{j+1}])\le \mathbb{P}_{X\in \cD^{j+1}}(X \in A\times (a,b])\le \epsilon + \epsilon\left(2^{n+1}\sqrt{n\ln\left(\frac{4K}{\delta}\right)}\right).\]
\end{itemize}

Now, for the sake of readability, we define the quantity $\zeta\coloneq \epsilon\left(2^{n+1}\sqrt{n\ln(\frac{4K}{\delta})}\right)$.
Hence, it holds:
\begin{align*}
    \bigg\lvert\mathbb{P}(X\in A\times ( \bigtimes_{i=j+1}^n [0,x_i]))- \sum_{B\in \mathcal{H}(A,\tilde x_{j+1})}\mathbb{P}(X\in B\times ( \bigtimes_{i=j+2}^n [0,x_i]))\bigg| \le \epsilon + \zeta \quad \forall j \in {0,\ldots,n-1},
\end{align*}
where we used Equation~\eqref{eq: aprox j}.

\paragraph{Recursive Approach to Reduce the Dimensionality of the Problem}

As we showed above, the probability of an hyperrectangle $A\in \mathcal{H}^j$ can be approximated with the probability of $|\mathcal{H}(A,\tilde x_{j+1})|$ hyperrectangles in $\mathcal{H}^{j+1}$, paying approximating price $\zeta + \epsilon$.  
Therefore, we can use the recursive structure of the estimator $\textsc{CGE-Rec}(\cdot,\cdot)$ itself to show the relationship between the error $\lvert\mathbb{P}(X\in A\times ( \bigtimes_{i=j+1}^n [0,x_i]))- \textsc{CGE-Rec}(A,x)|$ with $A\in \mathcal{H}^j$ and the error of $\lvert\mathbb{P}(X\in A\times ( \bigtimes_{i=j+1}^n [0,x_i]))- \textsc{CGE-Rec}(A,x)|$ with $A\in \mathcal{H}^{j+1}$.
Indeed,
\begin{align*}
\bigg\lvert\mathbb{P}(X\in A\times &( \bigtimes_{i=j+1}^n [0,x_i]))- \textsc{CGE-Rec}(A,x)\bigg|\\
&\le \bigg\lvert\mathbb{P}(X\in A\times ( \bigtimes_{i=j+1}^n [0,x_i]))- \sum_{B\in \mathcal{H}(A,\tilde x_{j+1})}\mathbb{P}(X\in B\times ( \bigtimes_{i=j+2}^n [0,x_i]))\bigg|\\
& \hspace{0.4cm}+ \bigg|\sum_{B\in \mathcal{H}(A,\tilde x_{j+1})}\mathbb{P}(X\in B\times ( \bigtimes_{i=j+2}^n [0,x_i]))- \textsc{CGE-Rec}(A,x)\bigg|\\
& \le (\epsilon + \zeta) + \bigg|\sum_{B\in \mathcal{H}(A,\tilde x_{j+1})}\mathbb{P}(X\in B\times ( \bigtimes_{i=j+2}^n [0,x_i]))- \textsc{CGE-Rec}(A,x)\bigg|\\
& = (\epsilon + \zeta) + \bigg|\sum_{B\in \mathcal{H}(A,\tilde x_{j+1})}\mathbb{P}(X\in B\times ( \bigtimes_{i=j+2}^n [0,x_i]))- \sum_{B\in \mathcal{H}(A,\tilde x_{j+1})}\textsc{CGE-Rec}(B,x)\bigg|\\
& \le (\epsilon + \zeta) + \sum_{B\in \mathcal{H}(A,\tilde x_{j+1})}\bigg|\mathbb{P}(X\in B\times ( \bigtimes_{i=j+2}^n [0,x_i]))-\textsc{CGE-Rec}(B,x)\bigg|.
\end{align*}

This observation, joint with the observation that, by \Cref{thm:logSize} $|\mathcal{H}(A,x_{j+1})|\le \log_2(|\mathcal{I}_0^{j+1}(A)|)\le \log_2(K)$, is enough to prove by induction the following.
For all $j=0,\ldots,n-1$, for all $A\in \mathcal{R}^j$:
\begin{align}
\label{eq: induction}
    \lvert\mathbb{P}(X\in A\times &( \bigtimes_{i=j+1}^n [0,x_i]))- \textsc{CGE-Rec}(A,x)|\le (\epsilon+\zeta)\sum_{k=0}^{n-1-j}(\log_2(K))^k + \zeta (\log_2(K))^{n-j}.
\end{align}

\paragraph{Proving Equation~\eqref{eq: induction} for ${j=n-1}$}

Let $A$ be in $\mathcal{R}^{n-1}$, then 
\begin{align*}
    \lvert\mathbb{P}(X\in &A\times [0,x_n]))- \textsc{CGE-Rec}(A,x)|\le |\mathbb{P}(X\in A\times (\tilde x_n,x_n])|+ |\mathbb{P}(X\in A\times [0,\tilde x_n]))-\textsc{CGE-Rec}(A,x)|\\
    &\le (\epsilon + \zeta) + |\mathbb{P}(X\in A\times [0,\tilde x_n])-\widehat{P}(A,\tilde x_n)|\\
    & \le (\epsilon +\zeta) + \zeta,
\end{align*}
where the last inequality holds by \Cref{theo: MC inef}.

\paragraph{Proving Equation~\eqref{eq: induction} for a Generic ${j}$}

Let $j$ be in $\{0,\ldots,n-2\}$.
Suppose Inequality~\eqref{eq: induction} holds for all $B\in \mathcal{R}^{j+1}$. Then, we can prove that Inequality~\eqref{eq: induction} holds also for all $A\in \mathcal{R}^j$.
\begin{align*}
\bigg\lvert\mathbb{P}(X\in A\times &( \bigtimes_{i=j+1}^n [0,x_i]))- \textsc{CGE-Rec}(A,x)\bigg|\\
& \le (\epsilon + \zeta) + \sum_{B\in \mathcal{H}(A,\tilde x_{j+1})}\bigg|\mathbb{P}(X\in B\times ( \bigtimes_{i=j+2}^n [0,x_i]))-\textsc{CGE-Rec}(B,x)\bigg|\\
& \le (\epsilon + \zeta) + |\mathcal{H}(A,\tilde x_{j+1})|\left((\epsilon+\zeta)\sum_{k=0}^{n-1-(j+1)}(\log_2(K))^k+ \zeta (\log_2(K))^{n-(j+1)}\right)\\
& \le (\epsilon + \zeta) + \log_2(K)\left((\epsilon+\zeta)\sum_{k=0}^{n-1-(j+1)}(\log_2(K))^k+ \zeta (\log_2(K))^{n-(j+1)}\right)\\
& \le (\epsilon + \zeta)+(\epsilon+\zeta)\sum_{k=1}^{n-1-j}(\log_2(K))^k + \zeta (\log_2(K))^{n-j}\\
& \le (\epsilon+\zeta)\sum_{k=0}^{n-1-j}(\log_2(K))^k + \zeta (\log_2(K))^{n-j}.
\end{align*}

\paragraph{Putting Everything Together}

Finally, let $j=0$, by Inequality~\eqref{eq: induction} :
\begin{align*}
    \bigg\lvert\mathbb{P}(X\in \bigtimes_{i=1}^n [0,x_i]))- \textsc{CGE-Rec}(\mathbbm{1},x)\bigg|&\le (\epsilon+\zeta)\sum_{k=0}^{n-1}(\log_2(K))^k + \zeta (\log_2(K))^{n}\\
    &\le (\epsilon+2\zeta)(\log_2(K))^n\\
    &= \left(\epsilon +  \epsilon\left(2^{n+2}\sqrt{n\ln\left(\frac{4K}{\delta}\right)}\right)\right)(\log_2(K))^n\\
    & \le \epsilon\left(2^{n+3}\sqrt{{n\ln\left(\frac{4K}{\delta}\right)}}\right)(\log_2(K))^n,
\end{align*}
where in the second inequality we upper bound the geometric series  $\sum_{k=0}^{n-1}r^k=(r^n-1)/(r-1)\le r^n$ for $r>2$.
\end{proof}

\subsection{Proof of \Cref{theo: main}}

Finally, we are ready to prove \Cref{theo: main}.  
The proof follows by combining \Cref{alg: map sync}, which constructs the representative hyperrectangles and corresponding estimates, with \Cref{alg: P tilde}, which aggregates them to form the final estimates.  
\Cref{alg: P tilde} is exactly the polynomial-time algorithm $\widetilde P(\cdot)$ returns by \Cref{theo: main}.

In particular, \Cref{alg: P tilde} achieves the desired precision level \( \epsilon \), provided that \Cref{alg: map sync} is initialized with an internal accuracy parameter  
\[
\epsilon' = \frac{\epsilon}{M}, \quad \text{where} \quad 
M = 2^{n+3}\sqrt{n\ln\!\left(\tfrac{4K}{\delta}\right)}\,(\log_2 K)^n.
\]
Then, by \Cref{thm:final}, the resulting estimator \( \widetilde{P}(\cdot) \) returned by \Cref{alg: P tilde} satisfies, with probability at least \(1-\delta\),
\[
\big| \mathbb{P}(X \le x) - \widetilde{P}(x) \big| \le \epsilon
\quad \forall\, x \in \mathcal{G}_K.
\]

Moreover, by \Cref{theo: MC inef},  \Cref{alg: map sync} requires at most
\[
\frac{2^{n-1} (4\log_2 K)^{n+2} M^3}{\epsilon^3}
= n^2\left(\log_2\!\left(\nicefrac K \delta\right)\right)^{\mathcal{O}(n)} \frac{1}{\epsilon^3}
\]
samples.  
This completes the proof of \Cref{theo: main}.

\section{Learning in Small Markets}

In this section, we show how our main result (\Cref{theo: main}) can be employed to derive results for learning problems in the small market model introduced in \Cref{sec:prelim}.

From \Cref{theo: main}, we immediately get a way of learning a uniform approximation of the trade probability $\mathbb{E}_{V \sim \mathcal{V}}[\Trade(V,P)]$ over an arbitrary uniform grid $\cG_K$ of resolution $1 / K$ ($K \in \mathbb{N}$). Crucially, \Cref{theo: main} shows that it is possible to achieve a sample complexity whose dependence from the dimensionality $n$ only affects logarithmic terms (in $K$). The main challenge is that, in small markets, we need to work over the entire set $[0,1]^n$ of fixed-price mechanism. In small markets, we are able to avoid making a bounded density assumption, by leveraging the Lipschitzness of the objective function $f$, the $1$-sided Lipschitzness of the trade probability, and the fact that we only need to learn an approximately-optimal fixed-price mechanism. 
%

The following is a simple corollary of \Cref{theo: main}.
\begin{corollary}\label{theo: pricingGrid}
 There exists an algorithm that, given an accuracy $\epsilon > 0$, a confidence $\delta \in (0,1)$, and a uniform grid $\mathcal{G}_K$ on $[0,1]^n$ of resolution $1 / K$ (with $K \in \mathbb{N}$), uses $\log(K/\delta)^{\mathcal{O}(n)}\frac{1}{\epsilon^3}$ queries and outputs a function $\widetilde P: \mathcal{G}_K \to [0,1]$ such that, with probability at least $1-\delta$, the following holds:\textnormal{
      \[\left|\mathbb{E}_{V \sim \mathcal{V}}[\Trade(V,p)] - \widetilde{P}({p})\right|\le \epsilon \quad \forall p \in \mathcal{G}_K,\]}
\end{corollary}

\begin{proof}
    The result follows from the observation that the problem can be recast as learning a uniform approximation of a CDF, through \Cref{eq:conversion}.
\end{proof}

Next, we exploit the $1$-sided Lipschitzness of the trade probability and the Lipschitzness of the objective function $f$. In particular, notice that, by increasing the prices offered to the sellers and decreasing those proposed to the buyers until reaching a point of the uniform grid, the trade probability can only increase ($1$-sided Lipschitness). At the same time, the objective function $f$ does \emph{not} decrease too much (Lipschitzness).
Formally, we can prove the following theorem.
\begin{theorem}\label{thm:pricing}
 There exists an algorithm that, given an accuracy $\epsilon>0$, a confidence $\delta \in (0,1)$, and a $L$-Lipschitz function $f: [0,1]^n \to [0,1]$, uses $\log(nL/(\delta\epsilon))^{\mathcal{O}(n)}\frac{1}{\epsilon^3}$ queries and outputs a solution $p^\star \in [0,1]^n$ such that, with probability at least $1-\delta$, the following holds:\textnormal{
      \[ \mathbb{E}_{V \sim \mathcal{V}}[\Trade(V,p^\star)] f(p^\star) \geq \max_{p \in [0,1]^n} \mathbb{E}_{V \sim \mathcal{V}}[\Trade(V,p)] f(p)- \epsilon.\]}
\end{theorem}

\begin{proof}
The algorithm instantiates an instance of the algorithm in \Cref{theo: pricingGrid} with accuracy $\epsilon'=\epsilon/3$ and $K=nL /\epsilon$ (with the same confidence $\delta$).
Then, the algorithm outputs 
\[p^\star\in \arg \max_{p \in \mathcal{G}_K} \widetilde{P}(p) f(p). \]

Then, the result essentially follows from the $1$-sided Lipschitzness of the trade probability and the Lipschitzness of the objective function $f$. Let $\tilde p\in \arg \max_{p\in [0,1]^n} \mathbb{E}_{V \sim \mathcal{V}}[\Trade(V,p)] f(p)$.
Then, it is possible to find $\overline p \in \cG_K$ such that $\mathbb{E}_{V \sim \mathcal{V}}[\Trade(V,\overline p)]\ge \mathbb{E}_{V \sim \mathcal{V}}[\Trade(V,\tilde p)]$ and $\lVert \overline p-\tilde p\rVert_1\le n/K$. This can be easily obtained by decreasing each buyer's price by at most ${1}/{K}$ and increasing each seller's price by at most ${1}/{K}$, until reaching a point on the grid $\mathcal{G}_K$.

Then, we get:
\begin{align*}
    \mathbb{E}_{V \sim \mathcal{V}}[\Trade(V,p^\star)] f(p^\star) & \ge \widetilde P(p^\star)f(p^\star)- \epsilon'\\
    & \ge \widetilde P(\overline p)f(\overline p)- \epsilon'\\
    & \ge  \mathbb{E}_{V \sim \mathcal{V}}[\Trade(V,\overline p)] f(\overline p)- 2\epsilon'\\
    & \ge  \mathbb{E}_{V \sim \mathcal{V}}[\Trade(V,\tilde p)] f(\tilde p) - 2\epsilon'-\frac{Ln}{K}\\
    & \ge \max_{p \in [0,1]^n}\mathbb{E}_{V \sim \mathcal{V}}[\Trade(V, p)] f(\tilde p)- \epsilon,
 \end{align*}
 where in the first and third inequalities we used \Cref{theo: pricingGrid}.
\end{proof}

Finally, \Cref{thm:pricing} combined with a standard explore-then-commit approach provides a $\tilde{\mathcal{O}}(T^{3/4})$ regret bound for the regret minimization problem in small markets. In particular, we set $\epsilon=T^{-1/4}$ and we learn a $\tilde{\mathcal{O}}(T^{-1/4})$ approximation of an optimal fixed-price mechanism by using $\tilde{\mathcal{O}}(1/\epsilon^3)= \tilde{\mathcal{O}}(T^{3/4})$ rounds. Then, we commit to that mechanism for the remaining rounds. Formally:
\begin{theorem}\label{theo: mainFixedPriceRegret}
    There exists an algorithm that, given a number of rounds $T \in \mathbb{N}$ and a $L$-Lipschitz function $f: [0,1]^n \to [0,1]$, achieves regret:\textnormal{
      \[ \textsf{REG}_T \coloneqq T \cdot \max_{p \in [0,1]^n} \mathbb{E}_{V \sim \mathcal{V}}[\Trade(V,p)] f(p) - \sum_{t\in [T]} \mathbb{E}_{V \sim \mathcal{V}}[\Trade(V,p_t)]  f(p_t) \le \log(nLT)^{\mathcal{O}(n)} T^{3/4}.\]}
\end{theorem}

\begin{proof}
Fix $\epsilon= T^{-1/4}$ and $\delta=1/T$.
  The algorithm uses $\log(nL/(\delta\epsilon))^{\mathcal{O}(n)}\frac{1}{\epsilon^3}$ rounds to learn a fixed-price mechanism $p^\star \in [0,1]^n$ that, thanks to \Cref{thm:pricing}, guarantees
  \[ \mathbb{E}_{V \sim \mathcal{V}}[\Trade(V,p^\star)] f(p^\star) \geq \max_{p \in [0,1]^n} \mathbb{E}_{V \sim \mathcal{V}}[\Trade(V,p)] f(p)- \epsilon\]
  with probability at least $1-\delta$.
  Then, the algorithm chooses $p_t=p^\star$ for the remaining rounds.
  Overall, with probability at least $1-\delta=1-1/T$, the regret is upper bounded as 
   \[ \textsf{REG}_T\le \log(nL/(\delta\epsilon))^{\mathcal{O}(n)}\frac{1}{\epsilon^3} + \epsilon T\le \log(nLT)^{\mathcal{O}(n)}T^{3/4}.\]
   The statement immediately follows.
\end{proof}

\section{Open Problems}

We conclude the paper discussing two interesting future directions.

\paragraph{Tight Lower Bounds}

While we have established tight sample complexity bounds (up to polylogarithmic factors) for constant-dimensional domains, it remains unclear whether these bounds are tight when the dimension $n$ is free. A interesting direction for future work is to fully characterize the dependency on $n$. For instance, determining whether an exponential dependence is fundamentally required or if better rates are achievable.

\paragraph{Tight Regret Bounds}

As a direct consequence, our results yield tight bounds on the sample complexity of learning fixed-price mechanisms in small markets. While these results imply novel regret bounds for fixed-price mechanisms, our regret rate remains $\tilde{\mathcal{O}}(T^{3/4})$. This does \emph{not} match the $\tilde{\Omega}(T^{2/3})$ lower bound for (single-dimensional) single-buyer pricing problems~\citep{kleinberg2003value}. Despite that, our approach matches the performance of a straightforward grid-based method, which runs a standard regret minimizer with bandit feedback over a uniform grid $\mathcal{G}_K$ of possible price vectors.
Hence, it remains unclear whether these bounds are tight for some fundamental problems, such as revenue maximization in bilateral trade.

Improving our bounds would require to mix the explore-then-commit phases into a single phase.
The most straightforward approach would be to convert our uniform approximation approach into a successive-arm-elimination-like algorithm.
This algorithm would work in phases $i$, and for each phase, keep track of the decisions $\mathcal{A}_{i} \subseteq \cG_K$ that are at most $\epsilon_i$ suboptimal (for an exponentially decreasing sequence of $\epsilon_i$).
Then, by only playing decisions in $\mathcal{A}_i$, it should achieve a uniform approximation of the reward of all the decisions in $\mathcal{A}_i$.
This would necessitate guaranteeing uniform convergence over a subset of ``promising'' prices, rather than over the fixed, uniform grid $\cG_K$ currently employed. 
Ideally, this approach would lead to a $\tilde {\mathcal{O}}(T^{2/3})$ regret bound. Unfortunately, it is unclear if our techniques can be extended to guarantee that a uniform approximation can be achieved on such an arbitrary set, or if a uniform approximation is possible only on a uniform grid $\cG_K$.

Designing matching upper and lower bounds on the regret for, \emph{e.g.}, revenue maximization, remains a exciting open problem.

\section*{Acknowledgments}

This work is funded by the FAIR (Future
Artificial Intelligence Research) project, funded by the NextGenerationEU program within the PNRRPE-AI scheme (M4C2, Investment 1.3, Line on Artificial Intelligence) and by the EU Horizon project
ELIAS (European Lighthouse of AI for Sustainability, No. 101120237)

\newpage
\printbibliography

@inproceedings{paseLeme,
  title={Pricing query complexity of revenue maximization},
  author={Leme, Renato Paes and Sivan, Balasubramanian and Teng, Yifeng and Worah, Pratik},
  booktitle={Proceedings of the 2023 Annual ACM-SIAM Symposium on Discrete Algorithms (SODA)},
  pages={399--415},
  year={2023},
  organization={SIAM}
}

@article{bubeck2008online,
  title={Online optimization in X-armed bandits},
  author={Bubeck, S{\'e}bastien and Stoltz, Gilles and Szepesv{\'a}ri, Csaba and Munos, R{\'e}mi},
  journal={Advances in Neural Information Processing Systems},
  volume={21},
  year={2008}
}

@article{slivkins2011multi,
  title={Multi-armed bandits on implicit metric spaces},
  author={Slivkins, Aleksandrs},
  journal={Advances in Neural Information Processing Systems},
  volume={24},
  year={2011}
}

@article{lunghi2025better,
  title={Better Regret Rates in Bilateral Trade via Sublinear Budget Violation},
  author={Lunghi, Anna and Castiglioni, Matteo and Marchesi, Alberto},
  journal={arXiv preprint arXiv:2507.11419},
  year={2025}
}

@article{naaman2021tight,
  title={On the tight constant in the multivariate Dvoretzky--Kiefer--Wolfowitz inequality},
  author={Naaman, Michael},
  journal={Statistics \& Probability Letters},
  volume={173},
  pages={109088},
  year={2021},
  publisher={Elsevier}
}

@article{dvoretzky1956asymptotic,
  title={Asymptotic minimax character of the sample distribution function and of the classical multinomial estimator},
  author={Dvoretzky, Aryeh and Kiefer, Jack and Wolfowitz, Jacob},
  journal={The Annals of Mathematical Statistics},
  pages={642--669},
  year={1956},
  publisher={JSTOR}
}

@article{massart1990tight,
  title={The tight constant in the Dvoretzky-Kiefer-Wolfowitz inequality},
  author={Massart, Pascal},
  journal={The annals of Probability},
  pages={1269--1283},
  year={1990},
  publisher={JSTOR}
}

@article{kleinberg2019bandits,
  title={Bandits and experts in metric spaces},
  author={Kleinberg, Robert and Slivkins, Aleksandrs and Upfal, Eli},
  journal={Journal of the ACM (JACM)},
  volume={66},
  number={4},
  pages={1--77},
  year={2019},
  publisher={ACM New York, NY, USA}
}

@inproceedings{kleinberg2003value,
  title={The value of knowing a demand curve: Bounds on regret for online posted-price auctions},
  author={Kleinberg, Robert and Leighton, Tom},
  booktitle={44th Annual IEEE Symposium on Foundations of Computer Science, 2003. Proceedings.},
  pages={594--605},
  year={2003},
  organization={IEEE}
}

@book{durrett2019probability,
  title={Probability: theory and examples},
  author={Durrett, Rick},
  volume={49},
  year={2019},
  publisher={Cambridge university press}
}

@article{cesa2024bilateral,
  title={Bilateral trade: A regret minimization perspective},
  author={Cesa-Bianchi, Nicol{\`o} and Cesari, Tommaso and Colomboni, Roberto and Fusco, Federico and Leonardi, Stefano},
  journal={Mathematics of Operations Research},
  volume={49},
  number={1},
  pages={171--203},
  year={2024},
  publisher={Informs}
}

@inproceedings{bernasconi2024no,
  title={No-regret learning in bilateral trade via global budget balance},
  author={Bernasconi, Martino and Castiglioni, Matteo and Celli, Andrea and Fusco, Federico},
  booktitle={Proceedings of the 56th Annual ACM Symposium on Theory of Computing},
  pages={247--258},
  year={2024}
}

@article{azar2022alpha,
  title={An $\alpha$-regret analysis of adversarial bilateral trade},
  author={Azar, Yossi and Fiat, Amos and Fusco, Federico},
  journal={Artificial Intelligence},
  volume={337},
  pages={104231},
  year={2024},
  publisher={Elsevier}
}

@inproceedings{cesa2021regret,
  title={A regret analysis of bilateral trade},
  author={Cesa-Bianchi, Nicol{\`o} and Cesari, Tommaso R and Colomboni, Roberto and Fusco, Federico and Leonardi, Stefano},
  booktitle={Proceedings of the 22nd ACM Conference on Economics and Computation},
  pages={289--309},
  year={2021}
}

@article{lunghi2025online,
  title={Online Two-Sided Markets: Many Buyers Enhance Learning},
  author={Lunghi, Anna and Castiglioni, Matteo and Marchesi, Alberto},
  journal={arXiv preprint arXiv:2503.01529},
  year={2025}
}

@inproceedings{babaioff2024learning,
  title={Learning to Maximize Gains From Trade in Small Markets},
  author={Babaioff, Moshe and Frey, Amitai and Nisan, Noam},
  booktitle={Proceedings of the 25th ACM Conference on Economics and Computation},
  pages={195--195},
  year={2024}
}

@inproceedings{aggarwal2024selling,
  title={Selling joint ads: A regret minimization perspective},
  author={Aggarwal, Gagan and Badanidiyuru, Ashwinkumar and D{\"u}tting, Paul and Fusco, Federico},
  booktitle={Proceedings of the 25th ACM Conference on Economics and Computation},
  pages={164--194},
  year={2024},
  doi={10.1145/3670865.3673520}
}

@article{di2025nearly,
  title={Nearly Tight Regret Bounds for Profit Maximization in Bilateral Trade},
  author={Di Gregorio, Simone and D{\"u}tting, Paul and Fusco, Federico and Schwiegelshohn, Chris},
  journal={arXiv preprint arXiv:2509.22563},
  year={2025}
}

\newpage

\appendix
\section*{Appendix}
\section{Omitted Proofs}

\lemmaEstimate*

\begin{proof}
    Let A be a $j$-dimensional hyperrectangle in $[0,1]^j$, i.e., $A\in \mathcal{H}^j$, defined by the intervals $\{(a^i,b^i]\}_{i=1}^j$, where $j\le n-1$, $b^i>a^i$ for all $i=1,\ldots,j$, and $w\in [0,1]$. 

\paragraph{Unbiasedness of $z_\tau$}

As a first step, we show that $z_\tau$ is an unbiased estimator of 
$\mathbb{P}(X \in A \times [0,w])$.

We distinguish two cases: $j=0$ and $j \neq 0$.

\paragraph{Case $j=0$.}
This case is trivial. 
Indeed, when $j=0$, we have $x_\tau = (w,1,\ldots,1)$ and $\lVert y_\tau \rVert_1 = 0$.
Hence,
\[
\mathbb{P}_{X\sim\cD^{j+1}}\!\left(X \in A \times [0,w]\right)
= \mathbb{P}_{X\sim \cD^1}(X \in [0,w])
= \mathbb{P}(X \le (w,1,\ldots,1))
= \mathbb{E}[\mathbb{I}(X_\tau \le x_\tau)]
= \mathbb{E}[z_\tau].
\]

\paragraph{Case $j \neq 0$.}
In this case, we use the inclusion–exclusion principle 
(see, e.g., \cite{durrett2019probability}) to decompose the target probability
$\mathbb{P}(X \in A \times [0,w])$ into a sum of \(2^j\) simpler terms.
Let $y\in \{0,1\}^{j}$ and $x(y)\in [0,1]^{n}$ be the point defined as
\[ 
 x_i(y) =
\begin{cases}
    a^i & \text{if } y_i = 1 \text{ and } i\in[j]\\
b^i & \text{if } y_i = 0 \text{ and } i\in[j]\\
w & \text{if } i=j+1\\
1 & \text{if } i\in \{j+2,\ldots,n\}.
\end{cases} 
\]

Then, by the inclusion--exclusion formula:
\begin{equation*}
\mathbb{P}_{X\sim \cD^{j+1}}\!\left(X \in A\times [0,w] \right)
= \sum_{y \in \{0,1\}^j} (-1)^{\lVert y \rVert_1}
\,\mathbb{P}\!\left(X \le x(y) \right).
\end{equation*}

Using this decomposition, we can verify the unbiasedness of $z_\tau$:
\begin{align*}
\mathbb{E}[z_\tau] &=   \mathbb{E}\!\left[2^j (-1)^{\lVert y_\tau \rVert_1}
    \mathbb{I}\!\left(X_\tau \le x_\tau\right)\right] \\[3pt]
&= \sum_{y \in \{0,1\}^j} \frac{1}{2^j} 
    \mathbb{E}\!\left[2^j (-1)^{\lVert y \rVert_1}
    \mathbb{I}\!\left(X_\tau \le x(y_\tau)\right)\right] \\[3pt]
&= \sum_{y \in \{0,1\}^j} (-1)^{\lVert y \rVert_1}
    \mathbb{P}_{X\sim\cD^{j+1}}\!\left(X \le x(y)\right) \\[3pt]
&= \mathbb{P}_{X\sim \cD^{j+1}}\!\left(X \in A \times [0,w]\right),
\end{align*}

where we use the fact that $y_\tau$ is sampled with uniform probability from $\{0,1\}^j$.

\paragraph{Concentration of the Empirical Mean}

Since $\{z_\tau\}_{\tau=1}^{N}$ is a sequence of independent random variables 
supported in $[-2^{j},\, 2^{j}]$, Hoeffding’s inequality implies that, 
with probability at least $1-\delta$,
\[
\left|
\frac{1}{N}\sum_{\tau=1}^N z_\tau
- \mathbb{P}_{X\sim \cD^{j+1}}\!\left(X\in A\times[0,w]\right)
\right|
\le 2^j \sqrt{\frac{\ln(2/\delta)}{2N}}
\le 2^{n-1} \sqrt{\frac{\ln(2/\delta)}{2N}},
\]
where the last inequality follows since $j \le n-1$.
\end{proof}

\end{document}